\documentclass[lettersize,journal]{IEEEtran}
\usepackage{amsmath,amsfonts}
\usepackage{algorithmic}
\usepackage{algorithm}
\usepackage{array}
\usepackage{amsmath,amssymb,amsfonts}
\usepackage[caption=false,font=normalsize,labelfont=sf,textfont=sf]{subfig}
\usepackage{textcomp}
\usepackage{stfloats}
\usepackage{url}
\usepackage{verbatim}
\usepackage{graphicx}
\usepackage{cite}
\usepackage{bm}
\usepackage{hyperref}
\hypersetup{hidelinks,
	colorlinks=true,
	allcolors=black,
	pdfstartview=Fit,
	breaklinks=true}

\newtheorem{proof}{Proof}
\newtheorem{ass}{\bf Assumption}
\newtheorem{remark}{\bf Remark}
\newtheorem{thm}{\bf Theorem}

\begin{document}

\title{Policy Learning for Nonlinear Model Predictive Control with Application to USVs}

\author{Rizhong~Wang, Huiping~Li,~\IEEEmembership{Senior Member,~IEEE,} Bin Liang,~\IEEEmembership{Senior Member,~IEEE}, Yang Shi,~\IEEEmembership{Fellow,~IEEE,}
	and~Demin Xu
	\thanks{R. Wang, H. Li and D. Xu are with School of Marine Science and Technology, Northwestern Polytechnical University, Xi'an 710072, China (e-mail: rizhongwang@mail.nwpu.edu.cn; lihuiping@nwpu.edu.cn; xudm@nwpu.edu.cn); B. Liang is with Tsinghua University (bliang@tsinghua.edu.cn);
	Y. Shi is with Department of Mechanical Engineering, University of Victoria (email:yshi@uvic.ca).}
	\thanks{National Natural Science Foundation of China (NSFC) under Grant 61922068, 61733014; Shaanxi Provincial Funds for Distinguished Young Scientists under Grant 2019JC-14; Aoxiang Youth Scholar Program under Grant.}}



\maketitle

\begin{abstract}
The unaffordable computation load of nonlinear model predictive control (NMPC) has prevented it for being used in robots with high sampling rates for decades. This paper is concerned with the policy learning problem for nonlinear MPC with system constraints, and its applications to unmanned surface vehicles (USVs), where the nonlinear MPC policy is learned offline and deployed online to resolve the computational complexity issue. A deep neural networks (DNN) based policy learning MPC (PL-MPC) method is proposed to avoid solving nonlinear optimal control problems online. The detailed policy learning method is developed and the PL-MPC algorithm is designed. The strategy to ensure the practical feasibility of policy implementation is proposed, and it is theoretically proved that the closed-loop system under the proposed method is asymptotically stable in probability. In addition, we apply the PL-MPC algorithm successfully to the motion control of USVs. It is shown that the proposed algorithm can be implemented at a sampling rate up to $5 Hz$ with high-precision motion control. The experiment video is available via:\url{https://v.youku.com/v_show/id_XNTkwMTM0NzM5Ng==.html}.
\end{abstract}


\section{Introduction}\label{section:introduction}

\IEEEPARstart{A}{s} one of the most advanced control technologies, model predictive control (MPC) has been widely used in process control\cite{mayne2000Industrial,joe2003Industrial}. Recently, MPC strategies have become popular to solve robot control problems because of its optimal control performance and the ability to handle complex system constraints \cite{Bemporad2018engines,AUV,UAV}. However, MPC requires solving an optimal control problem (OCP) periodically online, which is time consuming and cannot guarantee the real-time implementation in robot control applications. This brings great challenges to the application of MPC in the field of robotics\cite{Cannon2000,shi2017systems}.

In order to improve the online computational efficiency of OCP, several improved MPC algorithms have been proposed. In \cite{mark2004sqp}, the sequential quadratic programming (SQP) optimization strategy is adopted to decrease the computational time of OCP in nonlinear MPC (NMPC). The works in \cite{Carlo2019MIT} and \cite{Katz2019MIT} simplify the nonlinear quadruped robot model to a linear model. In that framework, the original NMPC optimization problem was replaced by a quadratic program (QP) problem, which can greatly reduce computational burden. In \cite{UAV}, a penalty term is added to the cost function to relax the inequality constraints of OCP, which improves the OCP efficiency.

Another approach to accelerate solution efficiency of OCP is to perform the online MPC optimization process offline. Explicit MPC is a typical method of is this regard \cite{Bemporad2009sempc,Ferreau2008sempc}. The works in \cite{Tondel2001Mempc} and \cite{Wen2009Aempc} obtain the explicit solution of the linear MPC optimal problem by solving the multi-parameter QP problem or using the piecewise affine function, respectively. In \cite{Kerrigan2018Aempc}, an approximate explicit MPC control policy is constructed using the barycentric interpolation method, and the system stability and feasibility of the algorithm is theoretically guaranteed. Note that all the above algorithms are only suitable for linear systems. For the non-convex OCP corresponding to the nonlinear system, the explicit MPC is no longer available because of the high computational complexity.

Recently, the learning MPC has become a promising approach to handle OCP based on machine learning \cite{Hewing2020Lmpc}. Since the control accuracy of MPC heavily relies on the accuracy of system model, the works in \cite{Aswani2013Lmpc,Limon2017Lmpc,Bouffard2012Lmpc} focused on learning the system model for MPC. Besides, machine learning methods can also assist the design of MPC controller. For example, the works in \cite{Bansal2017Lmpc,Piga2019Lmpc,Gros2020Lmpc} used a machine learning method to design the cost function. Alternatively, the work in \cite{Rosolia2018Lmpc} used the machine learning method to obtain a terminal region of MPC. Besides, machine learning algorithms can also learn the best system constraints of MPC\cite{Menner2021Lmpc,Chou2018Lmpc}. On the other hand, the MPC can provide expert data for machine learning algorithms. The work in \cite{Zhang2016Lmpc} used MPC to guide reinforcement learning (RL) training to improve the training efficiency of RL and guarantee the safety of the RL results.


Different from the traditional machine learning methods, deep neural networks (DNN) have strong capability to learn complex functions and policies. It can be implemented with faster computational speed of forward propagation after appropriate training, without solving complex optimization problems online. Therefore, a promising method to solve the real-time implementation issue of MPC is to use DNNs to learn the control policy offline and then deploy it online. In \cite{Zhang2021Lmpc}, the primal-dual DNN is used to learn the approximate policy of MPC. To improve the speed of solving large-scale linear MPC, the work in \cite{Chen2019Lmpc} used DNNs to learn the optimal solution, and the system stability and feasibility of optimization problem was guaranteed by primal active sets. The work in \cite{MichaelLearning} utilized the DNN to approximate control policy of robust MPC, and the feasibility and stability of the closed-loop system is also analyzed. In \cite{2108Chen}, constraint items are added to the traditional DNN to ensure that the approximate policy satisfies the constraints. However, most of the above works are developed for linear systems, which are not suitable for nonlinear systems such as robotic and vehicle systems. In addition, seldom have been implemented and verified via hardware systems. 

In this paper, we propose a new policy learning MPC (PL-MPC) scheme for constrained nonlinear systems using DNN and implement it to the motion control of USVs. The main contributions of the paper include:
\begin{itemize}
\item {We propose a deep supervised learning method to learn the policy of the constrained nonlinear MPC offline to greatly reduce the computational load. The detailed PL-MPC algorithm, the design of the DNN model and its training methods are developed, which provides a feasible tool to deploy nonlinear MPC for robotic control and planning in real-time application.}
\item {We conduct rigorous analysis of the proposed PL-MPC algorithm to make the proposed method theoretically valid. The proximal operator-based optimization and the quadratic penalty method are developed to ensure the practical feasibility of the proposed algorithm. With the PL-MPC algorithm, the closed-loop system is proved to be asymptotically stable in probability under mild conditions.}
\item {We implement the developed PL-MPC algorithm successfully to the motion control of an underactuated USV via lake experiments. The experimental results show that the PL-MPC algorithm can run up to the sampling rate of $5 Hz$, and the control performance is almost the same as the ideal NMPC in simulation. This verifies the effectiveness and advantage of the proposed policy learning MPC method for robotic.}
\end{itemize}

The rest of this paper is organized as follows: Section \ref{section:problem formulation} presents the preliminaries and Section \ref{section:LMPC policy} discusses the detailed PL-MPC algorithm. In Section \ref{section:feasibility and stability analysis}, the feasibility of the algorithm and the stability of the closed-loop system are presented. Then real-time experiments on USV motion control with the proposed method are conducted in Section \ref{section:experiments and comparisons}. Finally, Section \ref{section:conclusion} draws the conclusions.

Notations: In this paper, the following notations are used. The quadratic norm $\bm x^T \bm Q \bm x$ is denoted by $\left\| {{{\bm {x}}}} \right\|_{{Q}}^2$. $[\cdot]$ represents the integer operation, for example, $[x]$ denotes the integer part of $x$. $[x]^+$ denotes the function $\max(x,0)$, that is, if $x<0$, it has $[x]^+=0$, otherwise, $[x]^+=x$. If $X$ is a random variable, $P\{X>x\}$ denotes the probability of $X>x$. $E[X]$ and $V[X]$ represent the expectation and variance of the random variable $X$, respectively.

\section{Preliminaries}\label{section:problem formulation}

\subsection{System Model}\label{subsection:System Dynamics}

Consider a general discrete-time nonlinear dynamic system:

\begin{equation}\label{system dynamic}
\bm{x}(k+1) = \bm{f} (\bm{x}(k),\bm{u}(k)).
\end{equation}
Here, $\bm{x}(k)\in  \mathbb{R}^n$  and $\bm{u}(k) \in \mathbb{R}^m$ are the system state and the control input, respectively. The system state and control input are constrained by the compact sets $\mathcal{X}$ and $\mathcal{U}$ as follows:

\begin{equation}\label{constrain}
\begin{aligned}
&\mathcal{U}=\left\{\bm{u}(k)\in  \mathbb{R}^m|\bm{u}_{\min}\leq\bm{u}(k)\leq\bm{u}_{\max}\right\},\\
&\mathcal{X}=\left\{\bm{x}(k)\in  \mathbb{R}^n|\bm g(\bm{x}(k))\leq \bm{0}\right\}.
\end{aligned}
\end{equation}

We make the following assumptions about sets $\mathcal{U}$ and $\mathcal{X}$.

\begin{ass}\label{assumption:1}
$\mathcal{X}$ is a connected set, and $(\bm{0},\bm{0})$ is included within $\mathcal{U} \times \mathcal{X}$.
\end{ass}

To facilitate the controller design, the following assumption is made for the system in (\ref{system dynamic}).

\begin{ass}\label{assumption:2}
The function $\bm{f}(\cdot,\cdot):\mathbb{R}^n \times \mathbb{R}^m\to \mathbb{R}^n $ is second-order continuous differentiable, and $\bm{0} \in \mathbb{R}^n$ is an equilibrium point of the system, i.e., $\bm{f} (\bm{0},\bm{0}) = \bm{0}$.
\end{ass}

\subsection{Optimal Control Problem}\label{subsection:optimal control problem}

For MPC algorithm, the optimization object is a quadratic penalty on the error between predicted system state and control input with the target system state and control input. Without loss of generality, we choose the equilibrium point as the target system state and control input. The cost function is designed as 
\begin{equation}\label{cost function}
\begin{aligned}
&J(\bm{x}(k), \bar{\bm {u}}(\cdot;k))\\
=&\sum_{i=0}^{N-1} {F(\bar{\bm {x}}(k+i;k),\bar{\bm {u}}(k+i;k))}+V_f(\bar{\bm {x}}(k+N;k))\\
=&\sum_{i=0}^{N-1}({\left\| {{\bar{{\bm {x}}}}(k+i;k)} \right\|_{{Q}}^2 + \left\| {{{\bar{\bm {u}}}}(k+i;k))} \right\|_{{R}}^2})\\
&+\left\| {{{\bar{\bm {x}}}}(k+N;k)} \right\|_{{P}}^2,
\end{aligned}
\end{equation}
where $\bar{\bm {u}}(k+i;k)$ and $\bar{\bm {x}}(k+i;k)$ are control input sequence and system state sequence under cost function emanating from $\bm {x}(k)$. $N$ represents the step of the MPC algorithm in predicting the future state, which is called prediction horizon. $Q\geq0$, $R\geq0$ and $P\geq0$ represent the adjustable weighting matrices of MPC.

Then, OCP is formulated as follows:

{\bf Problem} $\mathcal{P}1$ :
\begin{equation*}\label{OCP1}
\begin{aligned}
	{{\bm {u}}^*}(\cdot;k)=\arg & \min_{{\bar{\bm {u}}}(\cdot;k)}  J(\bm{x}(k), \bar{\bm {u}}(\cdot;k))\\
	s.t.\quad& {{{\bar {\bm {x}}}}}(k+i+1;k) = \bm {f}({{\bar {\bm {x}}}}(k+i;k),{{\bar {\bm {u}}}}(k+i;k))\\
    & {{\bar {\bm {u}}}}(k+i;k) \in\mathcal{U} \\
    & {{\bar {\bm {x}}}}(k+i;k) \in\mathcal{X} \\
    & {{\bar {\bm {x}}}}(k+N;k)\in\mathcal{X}_f,
\end{aligned}
\end{equation*}
where $i=0,...,N-1$, $\mathcal{X}_f = \{\bm{x}\in  \mathbb{R}^n|\bm{x}^T \bm{P}\bm{x}\leq \alpha \}\subseteq \mathcal{X}$ is the terminal region, ${\bm {u}}^*(\cdot;k)=\{{\bm {u}}^*(k;k),{\bm {u}}^*(k+1;k),...,{\bm {u}}^*(k+N-1;k)\}$ is the optimal control input sequence.
 
The following standard assumption is made.

\begin{ass}\label{terminal layapunov}
For a system terminal state $\bm x \in \mathcal{X}_f$, there exists a local stabilizing control policy $\bm{\kappa} (\bm x)\in\mathcal{U}$  satisfying
\begin{equation}\label{detailed terminal lyapunov}
\begin{aligned}
 {V_f}({f}({{\bm {x}}},\bm{\kappa} ({\bm {x}}))) -{V_f}({{\bm {x}}})\leq 
 -F({\bm {x}},{\bm{\kappa} ({\bm {x}})}
 ),
\end{aligned}
\end{equation}
with $F(\bm x,\bm{\kappa} (\bm x))= {\left\| {{{\bm {x}}}} \right\|_{{Q}}^2 + \left\| {{{\bm{\kappa} (\bm x)}}} \right\|_{{R}}^2}$.
\end{ass}
Here, given the initial system state $\bm x(k)$ and a control sequence $\bm{U}(k)=\bar{\bm {u}}(\cdot;k)$, the state sequence $\bar {\bm {x}}(k+i;k)=\bm{\phi}(i,\bm{U}(k),\bm{x}(k))$ satisfies:

\begin{equation}\label{ODE}
\begin{aligned}
&\bm{\phi}(i,\bm{U}(k),\bm{x}(k))\\
&=\bm f(...\bm f(\bm{x}(k),\bm \mu(0,\bm{U}(k))),...,\bm \mu(i-1,\bm{U}(k))),
\end{aligned}
\end{equation}
where $i=0,...,N-1$ and $\bm \mu(i,\bm{U}(k))=\bar{\bm {u}}(k+i;k)$.

On this basis, we rewrite problem $\mathcal{P}1$ as the following compact form to facilitate discussions:

{\bf Problem} $\mathcal{P}2$ :
\begin{equation*}\label{OCP2}
\begin{aligned}
	\bm{U}^*(k)=\arg & \min_{\bm{U}(k)}  J(\bm{x}(k), \bm{U}(k))\\
	{\rm s.t.}\quad &\bm G(\bm x(k),\bm{U}(k))\leq \bm 0\\
    & {{\bar {\bm {x}}}}(k+N;k)\in\mathcal{X}_f.
\end{aligned}
\end{equation*}

According to (\ref{constrain}), $\bm G(\bm x(k),\bm{U}(k))$ is defined as:

\begin{equation}\label{C2}
\bm G(\bm x(k),\bm{U}(k))=
\left(
\begin{gathered}
\bm g(\bm{\phi}(0,\bm{U}(k),\bm{x}(k)))\\
\vdots\\
\bm g(\bm{\phi}(N-1,\bm{U}(k),\bm{x}(k)))\\
\bm \mu(0,\bm{U}(k))-\bm{u}_{max}\\
\vdots\\
\bm \mu(N-1,\bm{U}(k))-\bm{u}_{max}\\
\bm{u}_{min}-\bm \mu(0,\bm{U}(k))\\
\vdots\\
\bm{u}_{min}-\bm \mu(N-1,\bm{U}(k))\\
\end{gathered}
\right)
\end{equation}

\begin{figure}[!t]
\centering
\includegraphics[width=2.5in]{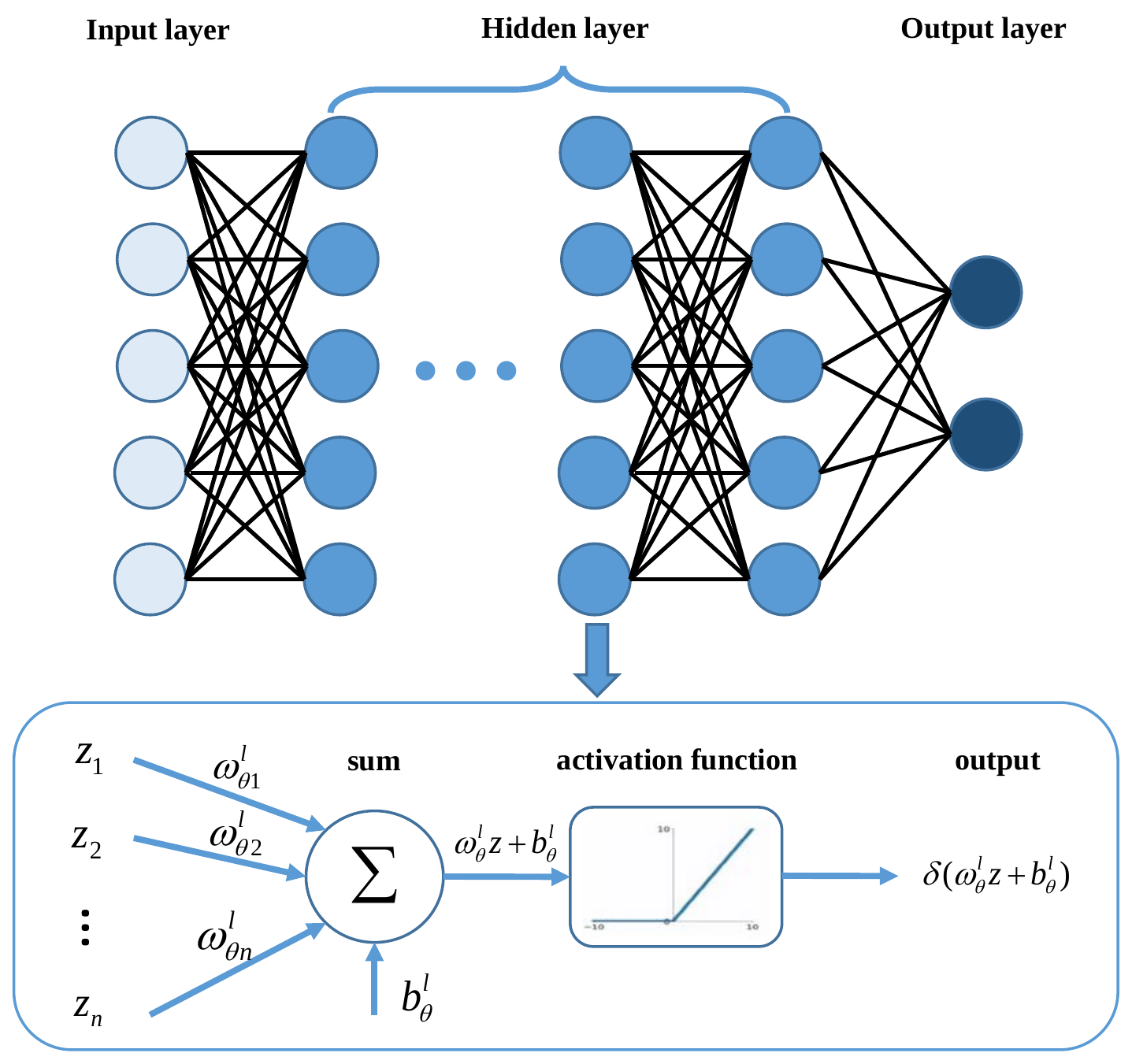}
\caption{Feedforward neural networks.}
\label{Neural Network}
\end{figure}

\subsection{Feedforward Neural Network}\label{subsection:deep neural network}

This paper adopts feedforward neural networks (FNN), a type of widely used DNN model, to learn the control policy of NMPC. As shown in figure \ref{Neural Network}, FNN provides a L-layer neural networks architecture, which forms a mapping from the input $\bm z$ to the output $\bm{f_\theta}(\bm z)$. The detailed formulation is as follows:
\begin{equation}\label{NN}
\begin{aligned}
	\bm {f_\theta}(\bm z)=\bm h_L(\bm h_{L-1}(\cdots \bm h_2(\bm h_1(\bm z)))),
\end{aligned}
\end{equation}
with the function
\begin{equation}\label{NN_}
\begin{aligned} 
\bm h_l(\bm z)=\delta(\bm \omega_{\theta}^{l} \bm z + \bm b_{\theta}^{l}).
\end{aligned}
\end{equation}

Here, $\delta$ is the activation function of DNN. The commonly used activation functions are: sigmoid function $\delta(x)=1/(1+e^{(-x)})$, tanh function $\delta(x)=(2/(1+e^{(-2x)})) - 1$ and ReLU function $\delta(x)=\max(0,x)$ \cite{Ian2016}.

To fitting a function $\bm f(\bm z)\in\mathcal{F}$ with the given training data $\{\bm z^{(i)},\bm f(\bm z^{(i)})\}_{i=1}^M$, the DNN optimizes its parameter $\bm\theta:= \{\bm \omega_{\theta}^{l}, \bm b_{\theta}^{l}\}_{l=1}^L$ in (\ref{NN_}) by minimizing a designed loss function $\mathcal{L}(\cdot)$. It has been proven that the multilayer FNN can approximate any finite discontinuous functions by an arbitrary accuracy\cite{Hornik1989,Leshno1993}. That is, DNN is a universal approximator for Borel measurable function, which is suitable for the NMPC policy learning task in this paper.

\section{POLICY Learning MPC}\label{section:LMPC policy}
\subsection{Policy Learning MPC}\label{subsection:learning approximate MPC policy}

In the framework of NMPC, Problem $\mathcal{P}2$ is recursively solved while the new system state $\bm x(k)$ is obtained, and only the first element of $\bm{U}^*(k)$ is applied to the system. Therefore, we can establish a function map between the system state $\bm x(k)$ and $\bm{U}^*(k)$ as follows:

\begin{equation}\label{CMap}
\begin{aligned}
\bm{U}^*(k)=\bm \pi(\bm x(k)).
\end{aligned}
\end{equation}

Different from the traditional NMPC algorithms utilizing the online optimization method, we use the DNN to fit $\bm \pi(\bm x(k))$ by $\bm{\pi_\theta}(\bm x(k))$ and denote
\begin{equation}\label{NMap}
\begin{aligned}
\bm{U_\theta}(k)=\bm{\pi_\theta}(\bm x(k)),
\end{aligned}
\end{equation}
where $\bm{U_\theta}(k)=\{\bm {u_\theta}(k;k),\bm {u_\theta}(k+1;k),...,\bm {u_\theta}(k+N-1;k)\}$. As introduced in Section \ref{section:problem formulation}, the parameter $\bm \theta$ of DNN is $\{\bm \omega_{\theta}^{l}, \bm b_{\theta}^{l}\}_{l=1}^L$. The detailed policy learning processes for NMPC are shown in Fig. \ref{figure20}.

\begin{figure}[!t]
\centering
\includegraphics[width=3.5in]{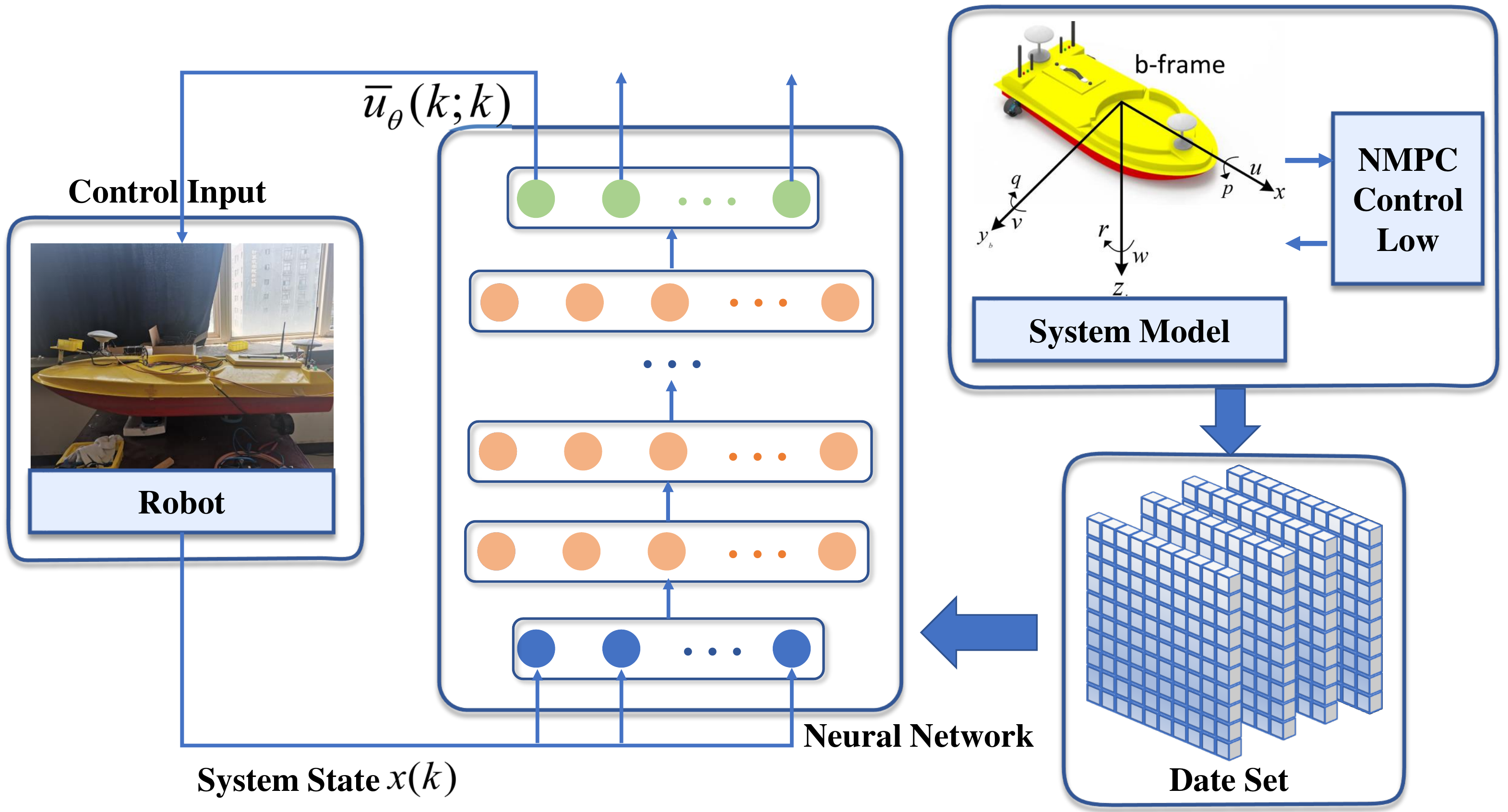}
\caption{policy Learning for MPC.}
\label{figure20}
\end{figure}

The loss function of deep learning is designed as
\begin{equation}\label{loss function}
\begin{aligned}
\mathcal{L}({\bm{U}^*(k)}, {\bm{U_\theta}(k)})=J(\bm{x}(k), \bm{U_\theta}(k))-J(\bm{x}(k), \bm{U}^*(k)),
\end{aligned}
\end{equation}
so that it can approximate $\bm \pi(\bm x(k))$ more accurately.

It is straightforward that the loss function $\mathcal{L}({\bm{U}^*(k)}, {\bm{U_\theta}(k)})\geq0$ because of the optimality property of $\bm{U}^*(k)$. By utilizing the training data set $\mathcal{D}=\{(\bm x(k), \bm{U}^*(k)\}_{k=1}^{M}$ and the loss function $\mathcal{L}({\bm{U}^*(k)}, {\bm{U_\theta}(k)})$, the supervised learning problem of $\bm{U_\theta}(k)$ is converted into the following optimization problem:

\begin{equation}\label{NN optimization}
\begin{aligned}
{{\bm {\theta}}^*}=\arg & \min_{{{\bm {\theta}}}}  \sum_{k=1}^M\mathcal{L}({\bm{U}^*(k)}, {\bm{U_\theta}(k)}).\\
\end{aligned}
\end{equation}

Considering the constraints of system defined in (\ref{constrain}), we obtain the following optimization problem:

{\bf Problem} $\mathcal{P}3$ :
\begin{equation*}\label{OCP3}
\begin{aligned}
	{{\bm {\theta}}^*}=\arg & \min_{{{\bm {\theta}}}}  \sum_{k=1}^M\mathcal{L}({\bm{U}^*(k)}, {\bm{U_\theta}(k)})\\
	s.t. \quad & \bm G(\bm x(k),\bm{U_\theta}(k))\leq \bm 0.
\end{aligned}
\end{equation*}

\subsection{Dual Optimization Learning Algorithm}\label{subsection:dual optimization learning algorithm}

Traditional learning algorithms only use the unconstrained optimization problems to update $\bm{\theta}$, which may lead to the failure of the learned control policy in practical implementation because of violating constraints. To deal with this problem, we fulfil the system state and input constraints in $\mathcal{P}3$ by means of Lagrange duality theory\cite{Fioretto2020}.

First, an augmented Lagrangian loss function is defined as follows:
\begin{equation}\label{Dual NN optimization}
\begin{aligned}
&\mathcal{L}({\bm{U}^*(k)}, {\bm{U_\theta}(k)},\bm v)\\
=&\mathcal{L}({\bm{U}^*(k)}, {\bm{U_\theta}(k)})+{\bm v}  [\bm G(\bm x(k),\bm{U_\theta}(k))]^+,
\end{aligned}
\end{equation}
where $\bm v = [v_1,...v_{3N}]$ is the dual variable associated with the constraint $\bm G(\bm x(k),\bm{U_\theta}(k))$.

Then, the dual form of problem $\mathcal{P}3$ is:

{\bf Problem} $\mathcal{P}4$ :
\begin{equation*}\label{OCP4}
\begin{aligned}
{\bm {\theta}}^*(\bm v)=\arg \min_{{{\bm {\theta}}}}\sum_{k=1}^M\mathcal{L}({\bm{U}^*(k)}, {\bm{U_\theta}(k)},\bm v),\\
{\bm {v}}^*=\arg \max_{\bm {v}}\min_{\bm{\theta}} \sum_{k=1}^M\mathcal{L}({\bm{U}^*(k)}, {\bm{U_\theta}(k)},\bm v).\\
\end{aligned}
\end{equation*}

Finally, for Problem $\mathcal{P}4$, we integrate the dual gradient descent algorithm and back propagation algorithm to solve it. The concrete steps of the dual optimization learning algorithm is described in Algorithm \ref{alg1}.  

Note that $\bm {v}$ and $\bm{\theta}$ are designed as an alternative update process in the framework of Algorithm \ref{alg1}. The procedure will be finished once the stopping criteria are satisfied, which means that  $\bm {v}$ may not converge to the optimal value at the end of iteration. Therefore, although the existence of the dual variable $\bm {v}$ encourages parameter $\bm{\theta}$ to converge to the approximate optimal value while satisfying the constraints in the process of updating, the non-optimal value  $\bm {v}$ would lead to a slight violation of the constraints in probability. This issue will be further discussed and tackled in Section \ref{subsection:feasibility analysis}.

\begin{algorithm}[ht]
	\caption{Dual Optimization Learning Algorithm}
	\begin{algorithmic}
        \REQUIRE Initial training dataset  $\mathcal{D}=\{(\bm x(k), \bm{U}^*(k)\}_{k=1}^{M}$, initial DNN parameter $\bm \theta$, initial optimal step size of DNN $\varepsilon$, initial dual variable $\bm v$, initial dual updated step size $\alpha=\{\alpha_0,\alpha_1...\}$.
           \FOR {each $j=0,1,2,...$}
		\WHILE {Loss function gradient greater than expected}
		\STATE {Randomly sample $m$ small batch samples $\{(\bm x(1), \bm{U}^*(1)),...,(\bm x(m), \bm{U}^*(m))\}\in \mathcal{D}$}
		\STATE Calculate the output of the DNN in system state $\bm x(k)$: $\bm{U_\theta}(k)= \bm{\pi_\theta}(\bm x_k)$
		\STATE Calculate $J(\bm{x}(k), \bm{U_\theta}(k))$ and $J(\bm{x}(k), \bm{U}^* (k))$ using (\ref{system dynamic}) and (\ref{cost function}). 
		\STATE Calculate $\mathcal{L}({\bm{U}^*(k)}, {\bm{U_\theta}(k)})$ using (\ref{loss function}).
           \STATE Convert the dual form $\mathcal{L}({\bm{U}^*(k)}, \bm{U_\theta}(k),\bm v^j)$ using to (\ref{Dual NN optimization}).
		\STATE Calculate the average loss functions of $m$ Samples $\frac{1}{m}\sum_{k=0}^m\mathcal{L}({\bm{U}^*(k)}, \bm{U_\theta}(k),\bm v^j)$.
		 
		\STATE Calculate the gradient of the average loss functions $\nabla_{\bm {\theta}}\frac{1}{m}\sum_{k=0}^m\mathcal{L}({\bm{U}^*(k)}, \bm{U_\theta}(k),\bm v^j)$ for parameter $\bm \theta$.
		\STATE Update parameters $\bm{\theta}\gets\bm \theta-\varepsilon \nabla_{\bm {\theta}}\frac{1}{m}\sum_{k=0}^m\mathcal{L}({\bm{U}^*(k)}, \bm{U_\theta}(k),\bm v^j)$
		\ENDWHILE
		\STATE Update the Lagrangian dual variable $v_{i}^{j+1} \gets v_i^{j}+\alpha_j \sum_{k=1}^n[\bm G_i(\bm x(k),\bm{U_\theta}(k))]^+ \quad i \in \{1,2,...,3N\}$ 
           \ENDFOR
	\end{algorithmic}
\label{alg1}
\end{algorithm}

\subsection{PL-MPC Algorithm}\label{subsection:LMPC algorithm}

The approximate control policy $\bm {\pi_\theta}(\bm x(k))$ is learned through the training dataset. In order to have a more comprehensive learning control policy, we need training samples that can cover the whole feasible region of MPC.

We use a rejection random sampling algorithm to sample the initial state ${\bm x}(0)$ of control trajectories in the feasible region. In order to reduce the similarity of the samplings and improve the coverage of samples in the feasible region, we determine whether to adopt the new sample by comparing the distance between the new sample $\bm x^k(0)$ and the existing one.

Assuming that there are $h$ existing samples $\{\bm x^1(0),\bm x^2(0)...,\bm x^h(0)\}$, their center of gravity can be expressed as:

\begin{equation}\label{gravity}
\begin{aligned}
\bm x^g(0)=\frac{1}{h}\sum_{i=1}^h \bm x^i(0).
\end{aligned}
\end{equation}

We can get the distance from the new sampling point $\bm x_0^k$ to the center of gravity $\bm x_0^g$ as:

\begin{equation}\label{dis}
\begin{aligned}
dis=\Vert\bm x^k(0)-\bm x^g(0)\Vert_2.
\end{aligned}
\end{equation}
If $dis>\tau_k$ is satisfied, $\bm x^k(0)$ is taken as the new initial point, otherwise the initial point will be sampled again. The minimum distance $\tau_k$ decreases with the increase of sampling number $k$ as
\begin{equation}\label{tau}
\begin{aligned}
\tau_k=\gamma^{[\frac{k}{K}]}\tau.
\end{aligned}
\end{equation}
Here, $\tau>0$ is a constant value used to filter the near sampling points and $\gamma$ is a discount factor satisfying $0<\gamma<1$. $K$ is a positive integer, which means that $\tau_k$ is updated every $K$ samples.

Next, the center of gravity is updated as follows:

\begin{equation}\label{up}
\begin{aligned}
\bm x^g(0)=\frac{1}{k}\bm x^k(0)+\frac{k-1}{k}\bm x^g(0).
\end{aligned}
\end{equation}

Using the above algorithm, we obtain a series of initial points ${\bm x}(0)$. In order to obtain training samples, we use the NMPC algorithm to get L-step control trajectory $\bm \tau(\bm x(0))=\{(\bm x(0),\bm{U}^*(0)),(\bm x(1),\bm{U}^*(1)),...,(\bm x(L),\bm{U}^*(L))\}$. We combine control trajectory from $H$ different initial states to get the training dataset
\begin{equation}\label{data set}
\begin{aligned}
\mathcal{D} = \bm \tau(\bm x^1(0)) \cup \bm \tau(\bm x^2(0)) \cup ...\cup \bm \tau(\bm x^H(0)).
\end{aligned}
\end{equation}
The same method is used to generate the test set for supervised learning.

There is a correlation between the sample data obtained by the above method, because the control trajectory is obtained in time sequence. In order to break the association of training data and make it independent and identically distributed (i.i.d), we use a data buffer to store and uniform randomly sample training data.

The training goal of the supervised learning algorithm is that the empirical error of the DNN loss function reaches the expected value. We reduce the empirical error of DNN training by increasing the number of training samples in the training dataset and increasing the number of training steps, until the experience error satisfies the condition.

When the DNN is trained, it will be deployed into the controller. To ensure stability, design the state region $\mathcal{X}_0 = \{\bm{x}\in  \mathbb{R}^n|\bm{x}^T \bm{Q}\bm{x}\leq \gamma\alpha \}\subseteq \mathcal{X}_f$, where $\gamma \in (0,1)$. If $\bm x(k) \in \mathcal{X}_0$, there exists $\bm{u}(k)=\bm{\kappa} (\bm {x}(k))$, otherwise use approximate control policy $\bm {\pi_\theta}(\bm x(k))$. The detailed learning algorithm is presented in Algorithm \ref{alg2}.

\begin{algorithm}[ht]
	\caption{PL-MPC Algorithm}
	\begin{algorithmic}
	    \REQUIRE Initialize training data buffer.\\
           \STATE\textbf{Offline:}
           \STATE Use rejection random sampling sampling method to randomly sample initial state $\bm x(0)$.
		\STATE Control trajectory $\bm \tau(\bm x(0))$ obtained by NMPC algorithm. 
		\STATE Store Control trajectory data in data buffer.
		\STATE Collect batch data from data buffer.
           \STATE Train DNN Using Algorithm \ref{alg1}. 
		\STATE Verify the learning effect by using empirical error
		\IF {The empirical error is less than the expected error}
           \STATE Stop training.
		\RETURN {The approximate control policy $\bm \pi_\theta$.}
		\ELSE
		\STATE Return to offline step 1 and continue training.
		\ENDIF
          \STATE \textbf{Online:}
          \FOR {$k=0,1,2,3...$} 
          \STATE Obtain the state of the controlled object $\bm x(k)$.
          \IF{$\bm x(k) \in \mathcal{X}_0$}
          \STATE Set control input $\bm{u}(k)=\bm{\kappa} (\bm {x}(k))$.
          \ELSE
          \STATE Set $N$ control sequences $\bm{U}(k)=\bar{\bm {u}}(\cdot;k)=\bm {\pi_\theta}(\bm x(k))$.
          \STATE Set control input $\bm {u}(k) = \bar{\bm {u}}(k;k)$.
          \ENDIF
          \ENDFOR
	\end{algorithmic}
\label{alg2}
\end{algorithm}

\section{FEASIBILITY HANDLING AND STABILITY ANALYSIS}\label{section:feasibility and stability analysis}
\subsection{Feasibility Handling}\label{subsection:feasibility analysis}

This subsection deals with the constraint satisfaction issue of Problem $\mathcal{P}3$ to ensure the feasibility of the proposed Algorithm \ref{alg2}. As mentioned at the end of Section \ref{subsection:dual optimization learning algorithm}, although the dual optimization learning algorithm provides an effective method to handle the constraints of PL-MPC, there exists a small chance that the approximate solution may violate constraints. In this particular case, we use the proximal operator to guarantee the feasibility of PL-MPC algorithm.

By adding a projection mapping layer on the basis of the original DNN, the proximal operator is able to translate the approximate solution $\bm {U_p}(k)$ generated by DNN into a high-quality feasible solution $\bm {U_\theta}(k)$. The projection mapping layer is designed as follows:

{\bf Problem} $\mathcal{P}5$ :
\begin{equation}\label{OPC5}
\begin{aligned}
\arg & \min_{{\bm U_\theta}(k)} \left\|{\bm {U_\theta}(k)-\bm {U_p}(k)}\right\|_2\\
s.t.\quad & \bm G(\bm x(k),\bm{U_\theta}(k))\leq \bm{0}. \\
\end{aligned}
\end{equation}

Generally, $\bm {U_p}(k)$ belongs to the feasible region by using the dual optimization learning algorithm, i.e., $\bm U_\theta(k)=\bm {U_p}(k)$. When the approximate solution $\bm {U_p}(k)$ does not exist in the feasible region, we can obtain the corresponding feasible solution $\bm {U_\theta}(k)$ by solving the optimization problem $\mathcal{P}5$.

Here, the problem $\mathcal{P}5$ is solved through exterior point penalty function methods. Define a penalty function as:

\begin{equation}\label{P_F}
\begin{aligned}
\bm {Q}(\bm {U_\theta}(k),\mu)=&\left\|{\bm {U_\theta}(k)-\bm {U_p}(k)}\right\|_2+\bm{\mu}([\bm G(\bm x(k),\bm{U_\theta}(k))]^+)^2,
\end{aligned}
\end{equation}
where $\bm{\mu} > \bm{0}$ is the penalty parameter. Obviously, as the increase of $\mu$, the penalty function $\bm {Q}(\bm {U_\theta}(k),\bm \mu)$ will be far away from the optimal value $\bm {Q}({\bm {U_\theta}}^*(k),\bm{\mu})$ if $\bm {U_\theta}(k)$ violates the constraint $\bm G(\bm x(k),\bm{U_\theta}(k))\leq \bm{0}$, which makes the unconstrained problem $\arg \min_{{\bm {U_\theta}(k)}} \bm {Q}(\bm {U_\theta}(k),\bm{\mu})$ equivalent to problem $\mathcal{P}5$. The detailed quadratic penalty method is presented in Algorithm \ref{alg3}.

\begin{algorithm}[ht]
	\caption{Quadratic Penalty Method}
	\begin{algorithmic}
	    \REQUIRE Initialize $\bm \mu_0=\bm 1$ and $\bm \gamma$, starting point $\bm {U_\theta}^0(k)=\bm {U_p}(k)$ and $\bm \xi > 0$.
		\FOR {$i=0,1,2,3...$}
		\STATE Optimize $\bm {Q}(\bm {U_\theta}(k),\bm {\bm{\mu}}_i)$ under $\bm {U_\theta}(k)$ starting at $\bm {U_\theta}^i(k)$, until $ \left\|\nabla_\mathbf u \bm {Q}(\bm {U_\theta}(k),{\bm{\mu}}_i)\right\|< \bm \xi$. 
		\IF {satisfy restraint condition}
		\STATE break.
		\ENDIF
		\STATE ${\bm{\mu}}_{i+1}=\bm \gamma{\bm{\mu}}_i$.
		\STATE Choose new starting point $\bm {U_\theta}^i(k)$.
		\ENDFOR
	\end{algorithmic}
\label{alg3}
\end{algorithm}

Now, by integrating the dual optimization learning algorithm and the proximal operator, there always exists a feasible solution for the problem  $\mathcal{P}3$.

\begin{remark}
Note that after the pre-processing of $\bm U_p(k)$ by the dual optimization learning algorithm, $\bm {U_p}(k)$ is not far away from $\bm {U_\theta}(k)$ in the case of constraint violations. Therefore, a small number of iterations are required to find $\bm {U_\theta}(k)$ in Algorithm \ref{alg3}, which can still obtain a fast solution without much computation resources.
\end{remark}

\subsection{Stability Analysis}\label{subsection:stability analysis}
In this subsection, we develop the stability result of the closed-loop system for the proposed learning algorithm.

We start with considering the loss function in (\ref{loss function}) with the learned control policy $\bm{U_\theta}$ designed in this paper. To improve the control performance of the proposed algorithm, we require the loss function to be less than a positive real number $\varepsilon>0$,
\begin{equation}\label{L_g_e}
\begin{aligned}
\mathcal{L}({\bm {U_\theta}}, {\bm {U^*}})< \varepsilon.
\end{aligned}
\end{equation}

In the policy learning procedure, due to the randomness of learning data and samples, the loss function $\mathcal{L}({\bm {U_\theta}}, {\bm {U^*}})$ is a random variable. In what follows, we develop a result ensuring (\ref{L_g_e}) holds in probability.

\begin{thm}\label{Theorem1}
Suppose the training data input $\bm {x}(k)$ satisfies the uniform distribution and the training data samples are sampled independently from this distribution. For the learned control policy $\bm{U_\theta}$ generated by Algorithm \ref{alg2} and the optimal one $\bm {U^*}$, given $\delta\in(0, 1]$ and $\varepsilon>0$, there exists a large number of training samples $M$, such that:

\begin{equation}\label{MK10}
\begin{aligned}
P\{\mathcal{L}({\bm {U_\theta}}, {\bm {U^*}})< \varepsilon\}\geq 1-\frac{R_{emp}(\bm{U_\theta}, \bm {U^*})}{(1-\sqrt{\frac{c-1}{M\delta}})\varepsilon},
\end{aligned}
\end{equation}
where $R_{emp}(\bm{U_\theta}, \bm {U^*})$ is the empirical error of $M$ training data samples defined as follows:

\begin{equation}\label{R_emp}
\begin{aligned}
&R_{emp}(\bm{U_\theta}, \bm {U^*})\\
=&\frac{1}{M}\sum_{k=1}^M(J(\bm{x}(k), \bm {U_\theta}(k))-J(\bm{x}(k), {\bm{U^*}(k)})),
\end{aligned}
\end{equation}
$c$ is a constant value satisfying $c\geq 1$.
\end{thm}

\begin{proof}
Considering the nonnegativity of $\mathcal{L}({\bm {U_\theta}}, {\bm {U^*}})$ and applying the Markov's Inequality to it for $ \varepsilon>0$, we have:

\begin{equation}\label{MK}
\begin{aligned}
&P\{\mathcal{L}({\bm {U_\theta}}, {\bm {U^*}})\geq\varepsilon \}\leq \frac{E[\mathcal{L}({\bm {U_\theta}}, {\bm {U^*}})]}{\varepsilon},
\end{aligned}
\end{equation}
where $E[\mathcal{L}({\bm {U_\theta}}, {\bm {U^*}})]$ is the generalization error of the learning algorithm defined as\cite{Daniel2019}:

\begin{equation}\label{R_gen}
\begin{aligned}
&E[\mathcal{L}({\bm {U_\theta}}, {\bm {U^*}})]
=R_{gen}(\bm{U_\theta}, \bm {U^*})\\
=&\int(J(\bm{x}, \bm{U_\theta})- J(\bm{x}, \bm{U^*}))dF(J(\bm{x}, \bm{U_\theta})|\bm{x}).
\end{aligned}
\end{equation}

Where $F$ is the probability density function of the distribution $\bm{x}(k)$. Then, by utilizing the inequality in (\ref{MK}), the probability of (\ref{L_g_e}) is bounded by:
\begin{equation}\label{MK2}
\begin{aligned}
P\{\mathcal{L}({\bm {U_\theta}}, {\bm {U^*}}) <\varepsilon\}
\geq 1-\frac{R_{gen}(\bm{U_\theta}, \bm {U^*})}{\varepsilon}.
\end{aligned}
\end{equation}

However, the generalization error $R_{gen}(\bm{U_\theta}, \bm {U^*})$ of the learning algorithm usually cannot be determined. To deal with this issue, we derive the upper bound of $R_{gen}(\bm{U_\theta}, \bm {U^*})$ in probability in the following.

First, we show that

\begin{equation}\label{dif}
\begin{aligned}
\frac{E[\mathcal{L}^2({\bm {U_\theta}}, {\bm {U^*}})]}{E^2[\mathcal{L}({\bm {U_\theta}}, {\bm {U^*}})]}\leq c.
\end{aligned}
\end{equation}

Because $\mathcal{L}({\mathbf {u_\theta}(\bm{x})}, {\mathbf {u^*}(\bm{x})})$ can be formulated as:
\begin{equation}\label{variance}
\begin{aligned}
V[\mathcal{L}({\bm {U_\theta}}, {\bm {U^*}})]= E[\mathcal{L}^2({\bm {U_\theta}}, {\bm {U^*}})]-E^2[\mathcal{L}({\bm {U_\theta}}, {\bm {U^*}})],
\end{aligned}
\end{equation}
it can be obtained that $c\geq1$. Since the variance $V[\mathcal{L}({\bm {U_\theta}}, {\bm {U^*}})]$ is bounded and greater than zero\cite{Yang2020}, there exists a positive integer $c\geq1$ such that the inequality in (\ref{dif}) holds. Since $\mathcal{L}({\bm {U_\theta}}, {\bm {U^*}})\geq 0$ and $\bm x(k)$ satisfy a uniform distribution, we have:

\begin{equation}\label{variance}
\begin{aligned}
E[\mathcal{L}^2({\bm {U_\theta}}, {\bm {U^*}})] \geq \frac{4}{3} E^2[\mathcal{L}({\bm {U_\theta}}, {\bm {U^*}})],
\end{aligned}
\end{equation}
$c$ can take the value $c=2$.

Then, substituting (\ref{dif}) into (\ref{variance}), we obtain:

\begin{equation}\label{dif1}
\begin{aligned}
V[\mathcal{L}({\bm {U_\theta}}, {\bm {U^*}})]\leq&(c-1) E^2[\mathcal{L}({\bm {U_\theta}}, {\bm {U^*}})]\\
\leq&(c-1)R_{gen}^2({\bm {U_\theta}}, {\bm {U^*}}).
\end{aligned}
\end{equation}
Next, by using the inequality in (\ref{dif1}), the upper bound of the variance of $R_{emp}(\bm{U_\theta}, \bm {U^*})$ can be calculated as:
\begin{equation}\label{var}
\begin{aligned}
& V[R_{emp}({\bm {U_\theta}}, {\bm {U^*}})]\\
=& V[\frac{1}{M}\sum_{k=1}^M (J(\bm{x}(k), \bm {U_\theta}(k))-J(\bm{x}(k), {\bm{U^*}(k)}))]\\
=& \frac{1}{M^2}\sum_{k=1}^MV[(J(\bm{x}(k), \bm {U_\theta}(k))-J(\bm{x}(k), {\bm{U^*}(k)}))]\\
=& \frac{1}{M}V[\mathcal{L}({\bm {U_\theta}}, {\bm {U^*}})]\\
\leq& {\frac{c-1}{M}}R_{gen}^2(\bm{U_\theta}, \bm {U^*}).
\end{aligned}
\end{equation}
In addition, the expectation of $R_{emp}({\bm {U_\theta}}, {\bm {U^*}})$ can be calculated as:
\begin{equation}\label{exp}
\begin{aligned}
&E[R_{emp}({\bm {U_\theta}}, {\bm {U^*}})]\\
&=E[\frac{1}{M}\sum_{k=1}^M(J(\bm{x}(k), \bm {U_\theta}(k))-J(\bm{x}(k), {\bm{U^*}(k)}))]\\
&=\frac{1}{M}\sum_{k=1}^ME[(J(\bm{x}(k), \bm {U_\theta}(k))-J(\bm{x}(k), {\bm{U^*}(k)}))]\\
&=\frac{1}{M}\sum_{k=1}^MR_{gen}(\bm{U_\theta}, \bm {U^*})=R_{gen}(\bm{U_\theta}, \bm {U^*}).
\end{aligned}
\end{equation}
By combining the inequality in (\ref{var}) and (\ref{exp}),
we can get:
\begin{equation}\label{chebyshev1}
\begin{aligned}
&P\{|R_{emp}(\bm{U_\theta}, \bm {U^*})-R_{gen}(\bm{U_\theta}, \bm {U^*})|<\sigma\}\\
&\geq 1-\frac{(c-1)R_{gen}^2(\bm{U_\theta}, \bm {U^*})}{\sigma^2M}
\end{aligned}
\end{equation}
where the Chebyshev inequality is utilized.

Setting $\sigma$ as $\sigma=\sqrt{\frac{c-1}{M\delta}}R_{gen}(\bm{U_\theta}, \bm {U^*})$ and substituting it into (\ref{chebyshev1}), we have:

\begin{equation}\label{chebyshev2}
\begin{aligned}
&P\{|R_{emp}(\bm{U_\theta}, \bm {U^*})-R_{gen}((\bm{U_\theta}, \bm {U^*})|\\
&<\sqrt{\frac{c-1}{M\delta}}R_{gen}(\bm{U_\theta}, \bm {U^*})\}\geq 1-\delta.
\end{aligned}
\end{equation}

Finally, considering that the empirical error is less than the generalization error\cite{Kil2002}, the upper bound of $R_{gen}(\bm{U_\theta}, \bm {U^*})$ satisfies the following:

\begin{equation}\label{R_gen1}
\begin{aligned}
P\{&R_{gen}(\bm{U_\theta}, \bm {U^*})\leq\\
 &(1-\sqrt{\frac{c-1}{M\delta}})^{-1}R_{emp}(\bm{U_\theta}, \bm {U^*})\}
\geq 1-\delta.
\end{aligned}
\end{equation}

When $\delta\rightarrow0$ and the number of training samples $M$ is large enough, the probability of $R_{gen}(\bm{U_\theta}, \bm {U^*})\leq(1-\sqrt{\frac{c-1}{M\delta}})^{-1}R_{emp}(\bm{U_\theta}, \bm {U^*})$ approaches to $1$. Applying (\ref{R_gen1}) into (\ref{MK2}), the inequality in (\ref{MK10}) is derived, which completes the proof.
\end{proof}
$\hfill\Box$

\begin{remark}
(\ref{var})  and (\ref{exp}) require the training sampling data $\bm{x_k}$ to be independently and identically distributed (i.i.d), which are widely used in machine learning algorithms\cite{nature,Machine}. In this paper, we set the data buffer and uniform random sampling to ensure the uniform distribution and i.i.d. property of the training sampling data.
\end{remark}

Next, the asymptotic stability in probability of the closed-loop system with the learned policy is presented.

\begin{thm}\label{Theorem2}
Suppose that Assumptions \ref{assumption:1}-\ref{terminal layapunov} hold, consider any state region $\mathcal{X}_0 = \subseteq \mathcal{X}_f$. Then the closed-loop system (\ref{system dynamic}) under the suboptimal control policy ${\bm {U}_\theta}$ obtained by Algorithm \ref{alg2} and Algorithm \ref{alg3} is asymptotically stable in probability.
\end{thm}

\begin{proof}
Consider that the system state is $\bm{x}(k)$, and the feasible suboptimal solution obtained from DNN is ${\bm {U}_\theta}(k) = [{\bm {u}_\theta}(k;k),...,{\bm {u}_\theta}(k+N-1;k)]$. Then we can get the state of the next moment $\bm{x}(k+1)=\bm{f}(\bm{x}(k),{\bm {u}_\theta}(k;k))$. For state $\bm{x}(k+1)$, there is a feasible control input sequence ${\bm {U}_\theta}(k+1) = [{\bm {u}_\theta}(k+1;k),...,{\bm {u}_\theta}(k+N-1;k),\bm v]$ due to  Assumptions \ref{terminal layapunov}, where $\bm v=\bm{\kappa} (\tilde {\bm {x}}(k+N;k))$. According to \cite{Bemporad}, we can get the following inequality:

\begin{equation}\label{cost function7}
\begin{aligned}
&J(\bm{x}(k+1), {\bm {U}^*(k+1)}) \leq J(\bm{x}(k+1), {\bm {U}_\theta}(k+1))\\
&= J(\bm{x}(k), {{\bm {U}_\theta}(k)})
+({\left\| {{\tilde{{\bm {x}}}}(k+N;k)} \right\|_{{Q}}^2 + \left\| {\bm v} \right\|_{{R}}^2})\\
&-(\left\| {{{{\bm {x}}}}(k;k)} \right\|_{{Q}}^2 + \left\| {{{{\bm {u}}}}(k;k))} \right\|_{{R}}^2)\\
&+\left\| {{{\tilde{\bm {x}}}}(k+N+1;k)} \right\|_{{P}}^2-\left\| {{{\tilde{\bm {x}}}}(k+N;k)} \right\|_{{P}}^2.
\end{aligned}
\end{equation}

Here, $\bm {U}^*(k+1)= [\bar{\bm {u}}^*(k+1;k),...,\bar{\bm {u}}^*(k+N;k)]$ is the optimal solution. According to Assumption 3, it follows that:

\begin{equation}\label{cost function8}
\begin{aligned}
&J(\bm{x}(k+1), {\bm {U}^*(k+1)}) -  J(\bm{x}(k), {{\bm {U}_\theta}(k)})\leq\\
&-(\left\| {{{{\bm {x}}}}(k)} \right\|_{{Q}}^2 + \left\| {{{{\bm {u}}}}(k))} \right\|_{{R}}^2).
\end{aligned}
\end{equation}

Using Theorem \ref{Theorem1}, when the sampling number $M$ is large enough, then there is small enough $\varepsilon>0$, such that $\mathcal{L}({\bm {U_\theta}}, {\bm {U^*}})<  \varepsilon \leq \mathop{max}\limits_{\bm x \in \mathcal{X}_0}\left\| {{{\bm {x}}}} \right\|_{{Q}}^2$. Then we can get $J(\bm{x}, {\bm {U}_\theta}) < J(\bm{x}, {\bm{U^*}})+\mathop{max}\limits_{\bm x \in \mathcal{X}_0}\left\| {{{\bm {x}}}} \right\|_{{Q}}^2$. Since $\left\| {{{{\bm {u}}}}(k))} \right\|_{{R}}^2\geq0$, the inequality (\ref{cost function8}) can be transformed into:

\begin{equation}\label{cost function9}
\begin{aligned}
&J(\bm{x}(k+1), {\bm {U}^*(k+1)}) -  J(\bm{x}(k), {{\bm {U}^*}(k)})<\\
&\mathop{max}\limits_{\bm x \in \mathcal{X}_0}\left\| {{{\bm {x}}}} \right\|_{{Q}}^2-\left\| {{{{\bm {x}}}}(k)} \right\|_{{Q}}^2 .
\end{aligned}
\end{equation}

For $\bm {x}(k) \notin \mathcal{X}_0$, we can get\cite{Johansen}:

\begin{equation}\label{cost function10}
\begin{aligned}
J(\bm{x}(k+1), {\bm {U}^*(k+1)}) -  J(\bm{x}(k), {{\bm {U}^*}(k)})<0.
\end{aligned}
\end{equation}

This indicates that the suboptimal control law is asymptotically stable with respect to a set  $\mathcal{X}_0 = \{\bm{x}\in  \mathbb{R}^n|\bm{x}^T \bm{P}\bm{x}\leq \gamma\alpha\}\in \mathcal{X}_f$. From this we can conclude that the closed-loop system under suboptimal control policy obtained by Algorithm \ref{alg2} is asymptotically stable, because for $\bm x(k) \in \mathcal{X}_0$, $\bm{u}(k)=\bm{\kappa} (\bm {x}(k))$ is a local stabilizing control policy for the closed-loop system.

Finally, from Theorem \ref{Theorem1} it follows that, $P\{\mathcal{L}({\bm {U_\theta}}, {\bm {U^*}}) <\mathop{max}\limits_{\bm x \in \mathcal{X}_0}\left\| {{{\bm {x}}}} \right\|_{{Q}}^2\}\rightarrow1$. As a result, the probability that the system is asymptotically stable will also be close to $1$. The proof is completed.

\end{proof}
$\hfill\Box$


\begin{figure}[!t]
\centering
\includegraphics[width=3.2in]{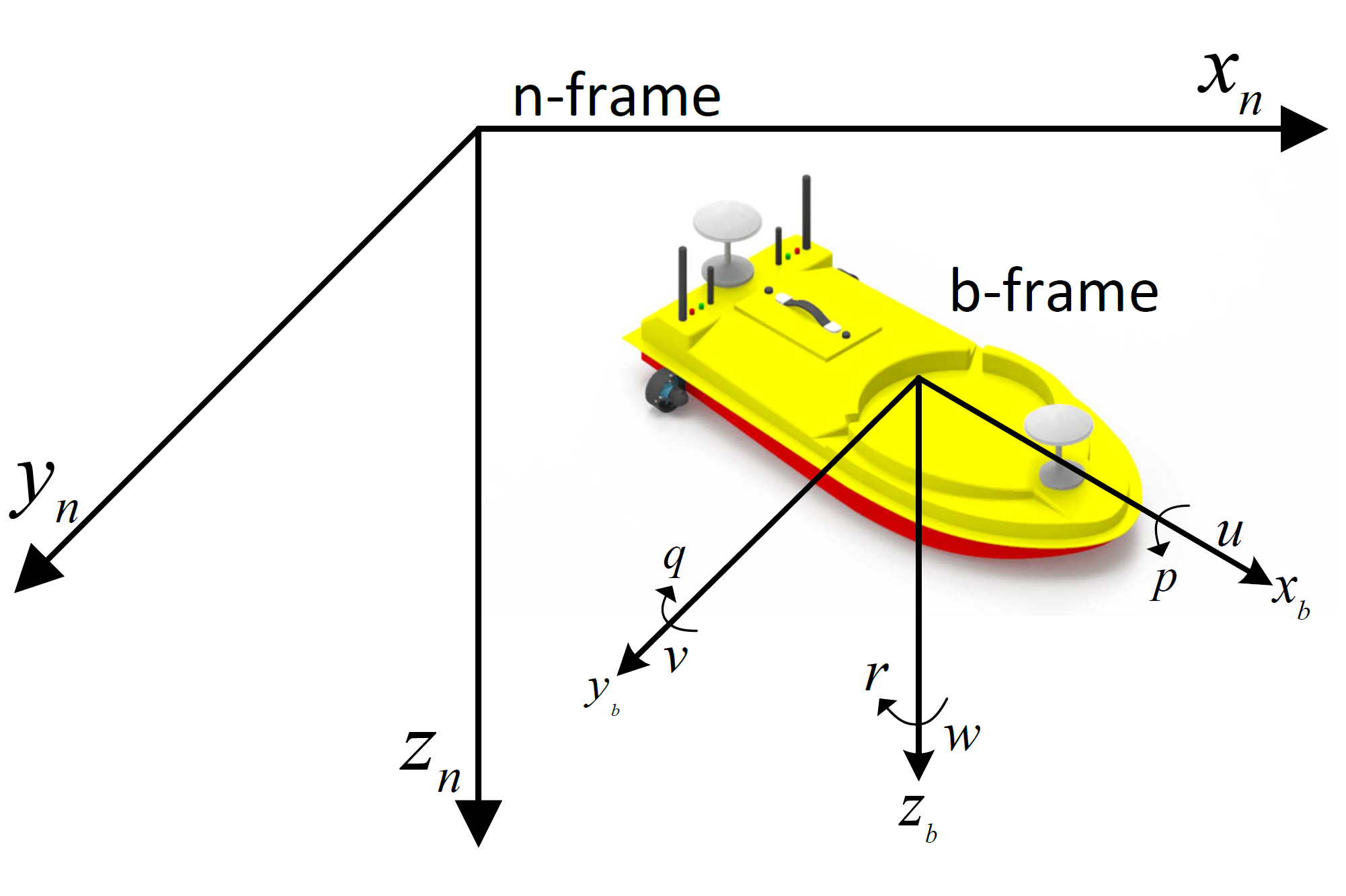}
\caption{State definition of unmanned surface vessel.}
\label{figure0}
\end{figure}

\section{IMPLEMENTATION TO USVs}\label{section:experiments and comparisons}

In this section, we implement the proposed PL-MPC algorithm for an underactuated USV.

\subsection{USV Dynamics}\label{subsection:simulation experiment}

\begin{figure*}[!t]
\centering
\subfloat[]{\includegraphics[width=2.8in]{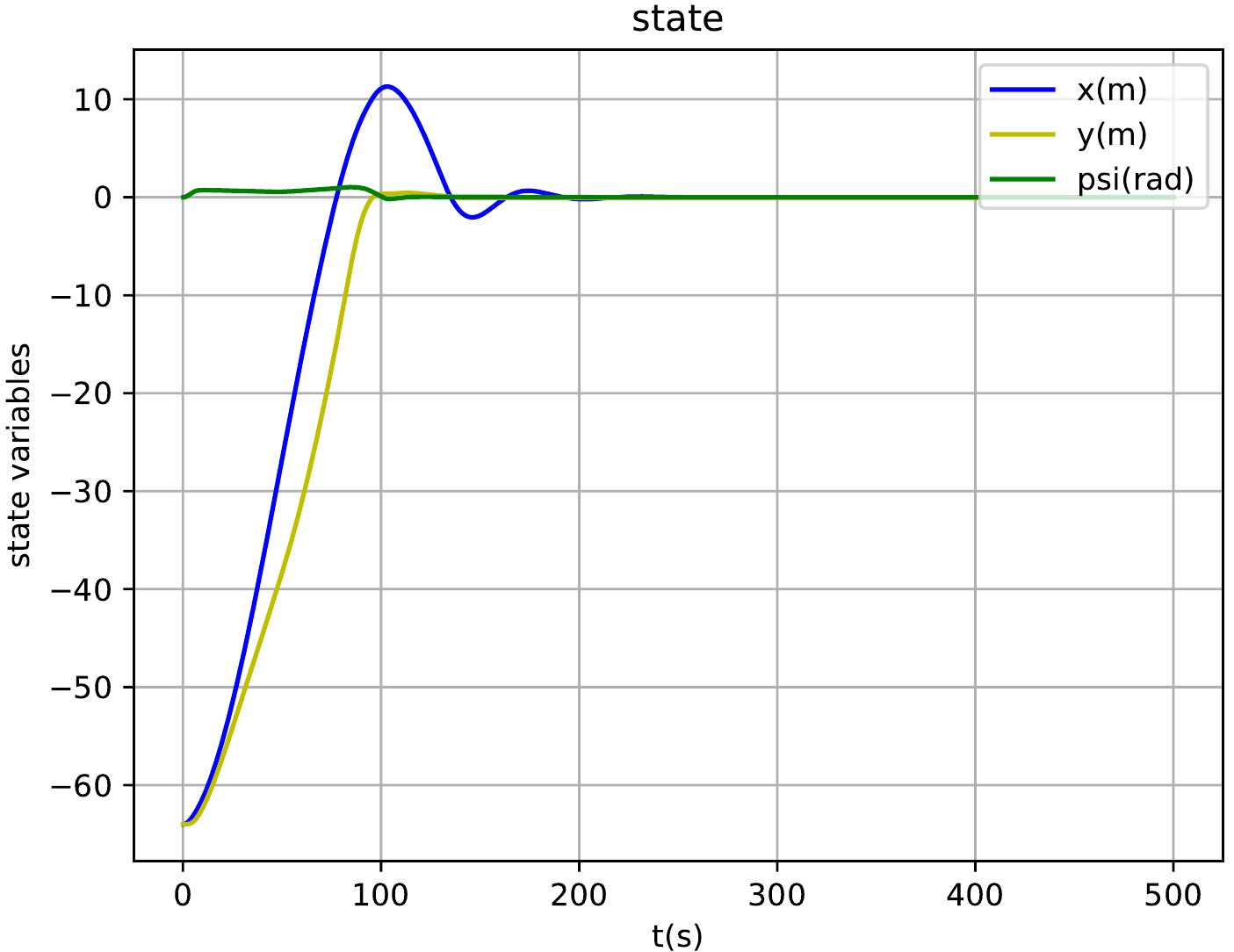}
\label{figure1}}
\hfil
\subfloat[]{\includegraphics[width=3in]{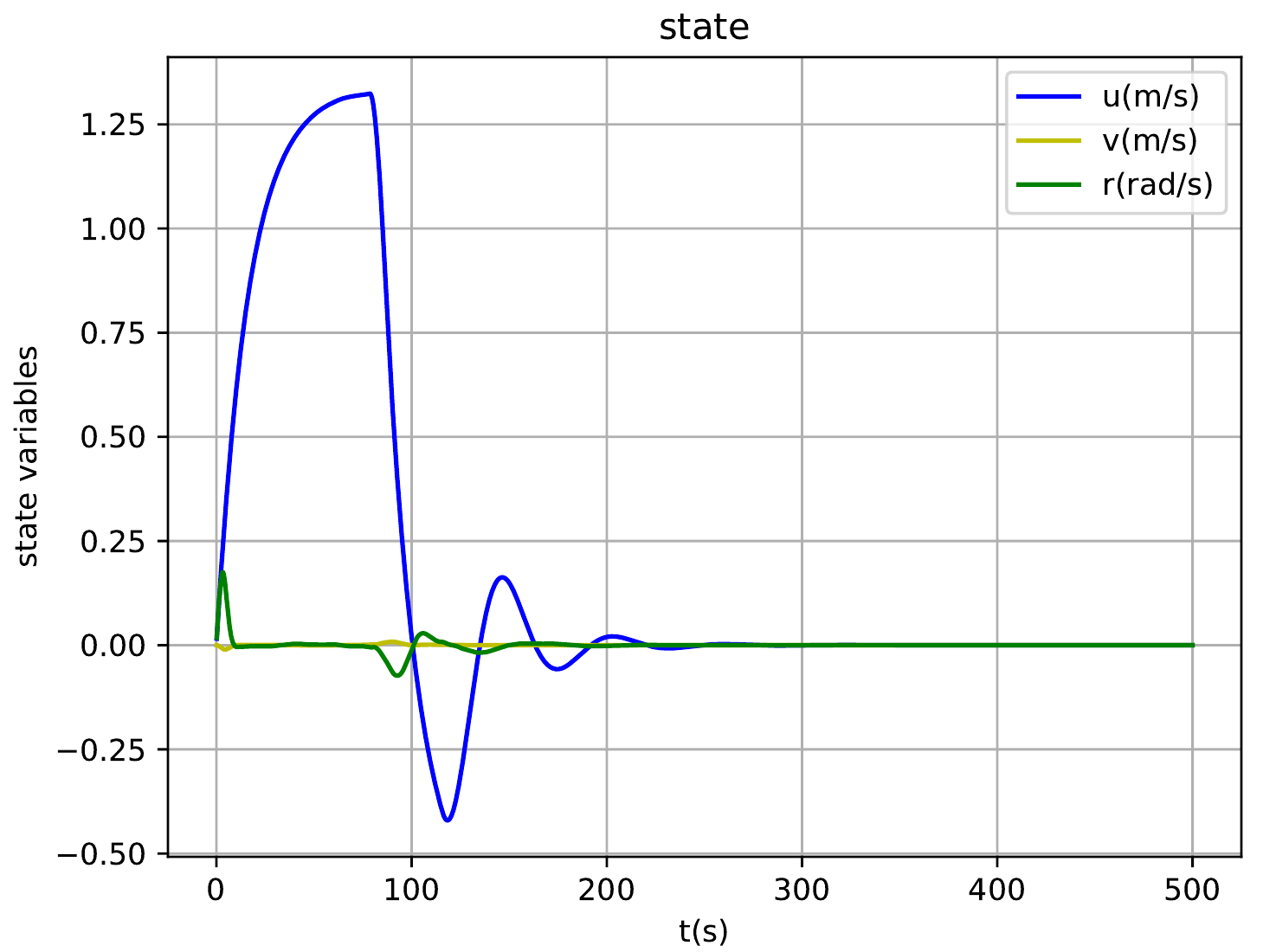}
\label{figure2}}
\caption{State trajectory with initial system state $[-64,-64,0,0,0,0]^T$.  (a) Position trajectory. (b) Velocity trajectory.}
\vfill
\subfloat[]{\includegraphics[width=2.8in]{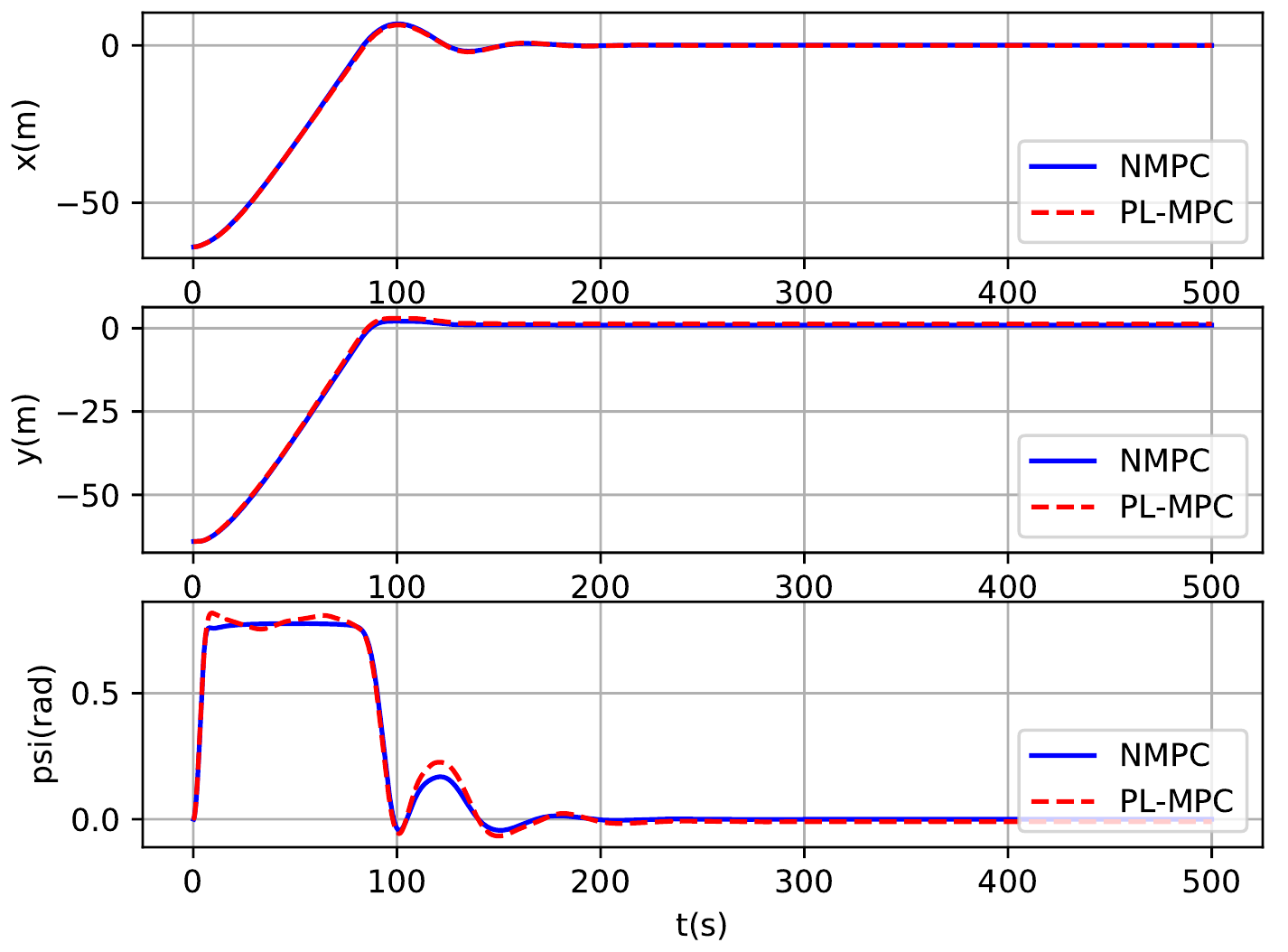}
\label{figure3}}
\hfil
\subfloat[]{\includegraphics[width=2.8in]{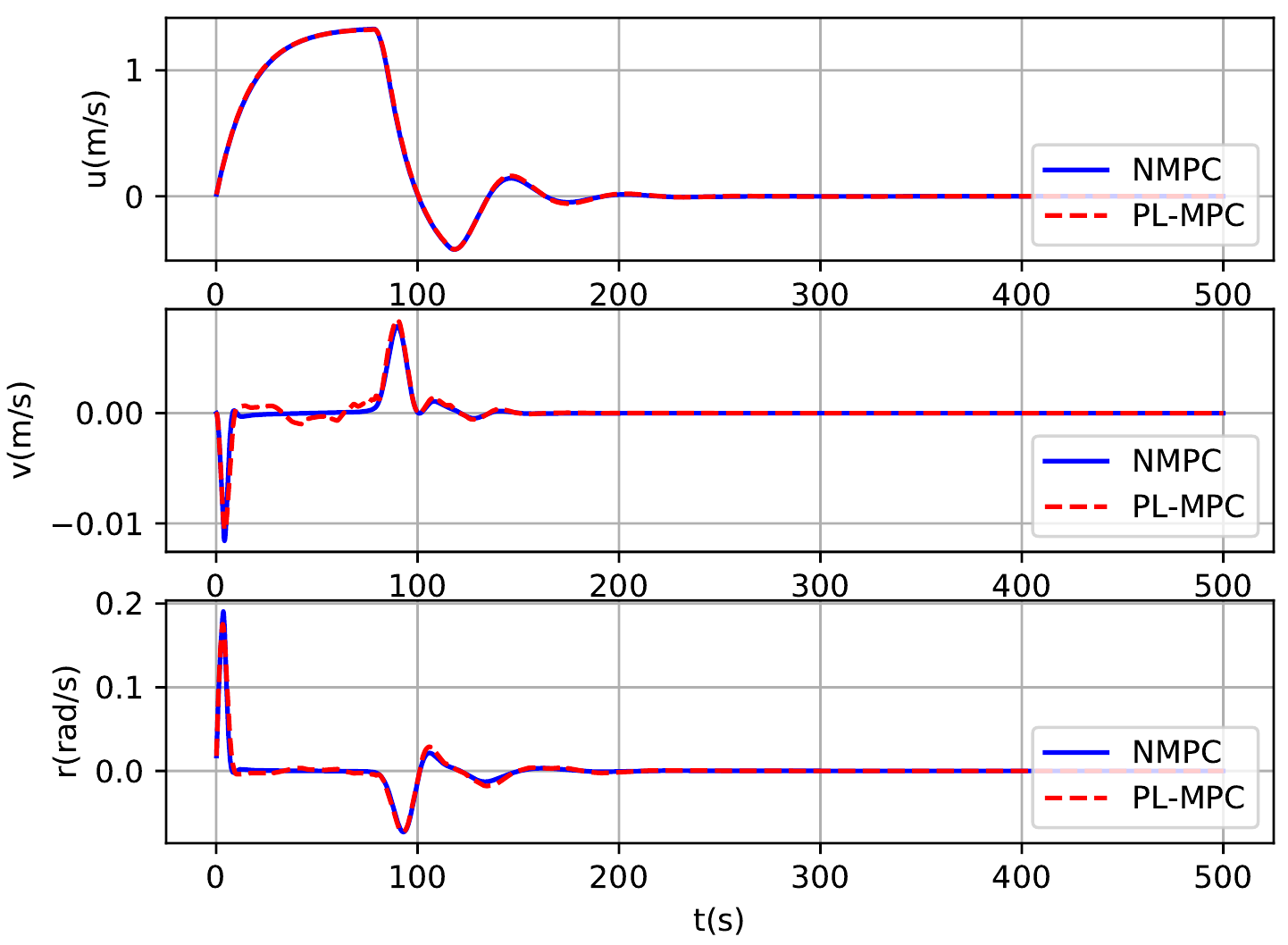}
\label{figure4}}
\caption{Comparison of PL-MPC and NMPC algorithms.  (a) Position comparison. (b) Velocity comparison.}
\vfill
\subfloat[]{\includegraphics[width=2.8in]{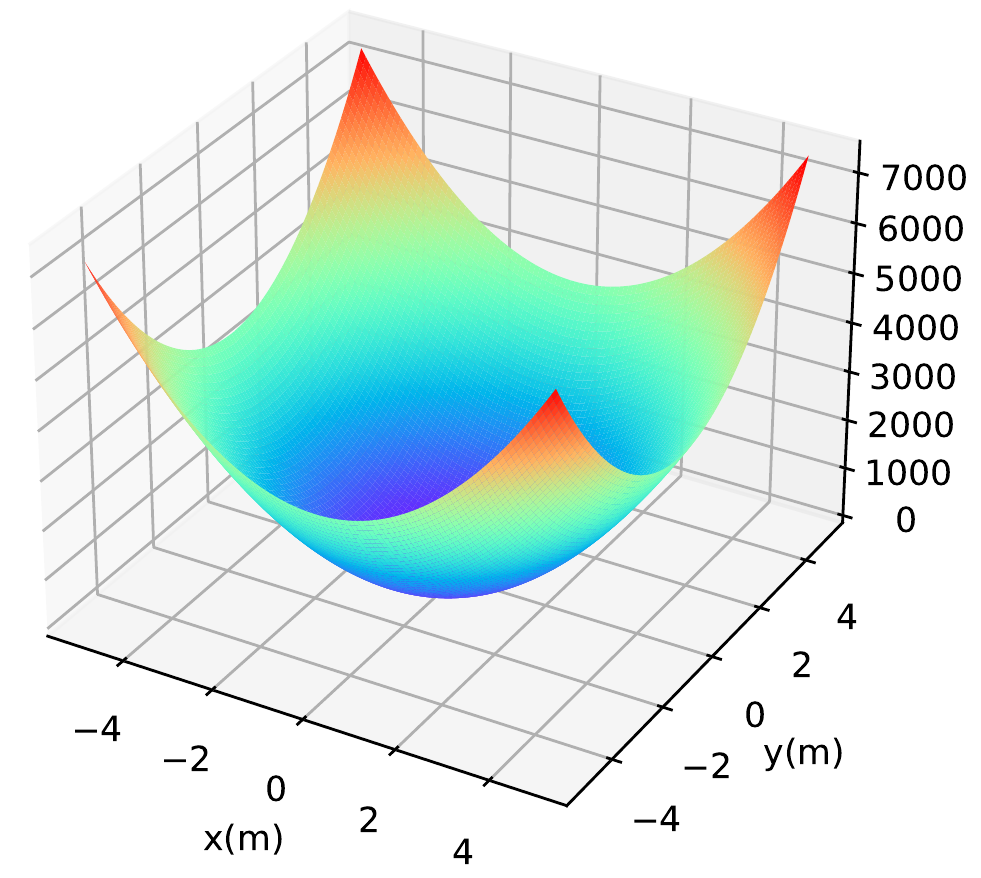}
\label{figure17}}
\hfil
\subfloat[]{\includegraphics[width=3.2in]{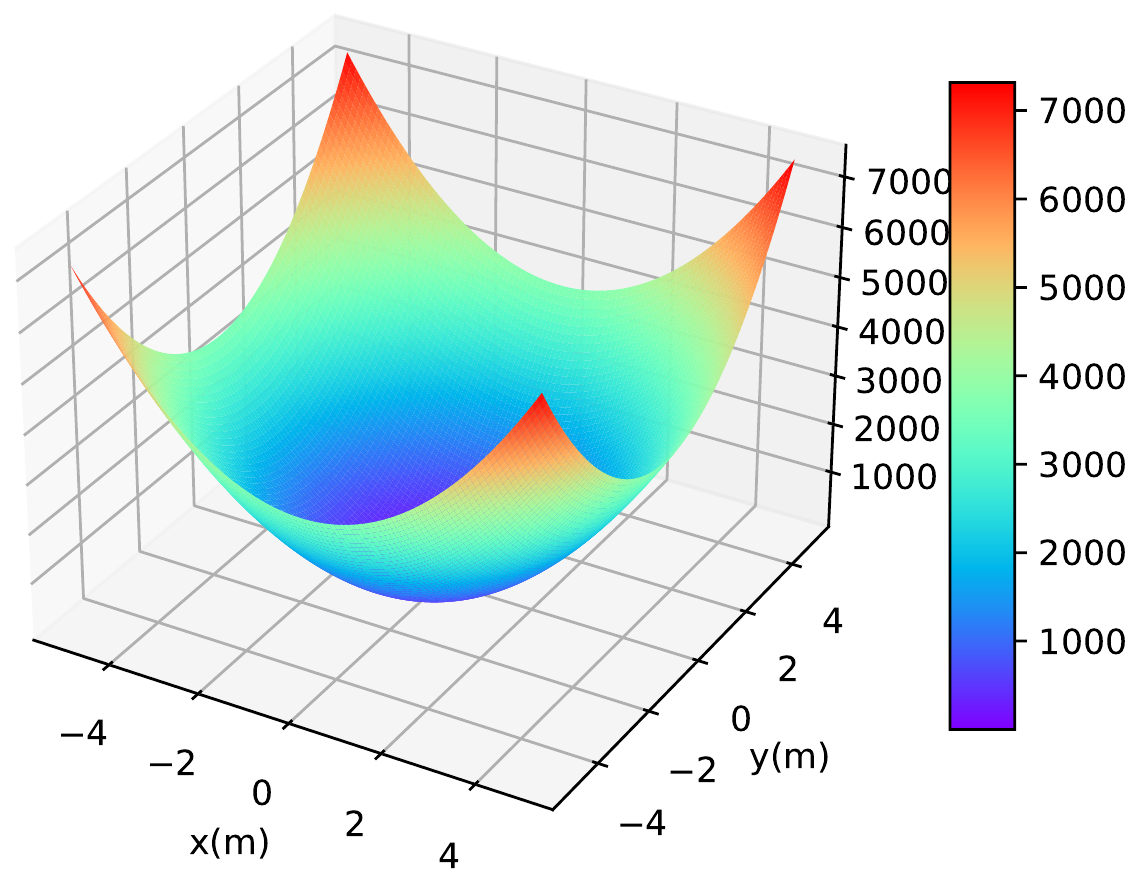}
\label{figure18}}
\caption{Cost function with initial system state $[-5,-5,0,0,0,0]^T$ to $[5,5,0,0,0,0]^T$.  (a)  Optimal cost function of NMPC. (b)  suboptimal cost function of PL-MPC.}
\end{figure*}

To generate training data and validate performance of the proposed algorithm, we identify the model for the USV shown in Fig. \ref{figure0}. The USV is driven by two rear thrusters, which are $T_l$ and $T_r$, respectively. The control input $F = T_l + T_r$, $M = (T_l - T_r)*d_l$, where $d_l$ is axis distance of USV equals to $0.277m$. The state space of the USV is from the north-east-down frame (n-frame) to the body frame (b-frame). It has the nonlinear discrete-time system dynamics as follows:
\begin{align}\label{example_system}
	\begin{bmatrix}
		x_{k+1}\\y_{k+1}\\{\psi}_{k+1}\\u_{k+1}\\v_{k+1}\\r_{k+1}
	\end{bmatrix} =
	\begin{bmatrix}
		x_{k}\\y_{k}\\{\psi}_{k}\\u_{k}\\v_{k}\\r_{k}
	\end{bmatrix} + \Delta T
	\begin{bmatrix}
		\Delta x_{k}\\ \Delta y_{k}\\ \Delta {\psi}_{k}\\ \Delta u_{k}\\ \Delta v_{k}\\ \Delta r_{k}
	\end{bmatrix},
\end{align}
where $[\Delta x_{k}, \Delta y_{k}, \Delta {\psi}_{k}, \Delta u_{k}, \Delta v_{k}, \Delta r_{k}]^T$ are described as:

\begin{align}\label{example_system1}
	\begin{bmatrix}
		\Delta x_{k}\\ \Delta y_{k}\\ \Delta {\psi}_{k}
	\end{bmatrix} =
	\begin{bmatrix}
		\cos{\psi}_{k} & -\sin{\psi}_{k} & 0\\ \sin{\psi}_{k} & \cos{\psi}_{k} & 0\\0 & 0 & 1
	\end{bmatrix}
	\begin{bmatrix}
		 u_{k}\\ v_{k}\\ r_{k}
	\end{bmatrix},
\end{align}

\begin{equation}\label{example_system2}
\left\{
\begin{aligned}
\Delta u_{k} & = (m_{22}v_{k}r_{k}-d_{11}u_{k}+F)/m_{11}\\
\Delta v_{k} & = (-m_{11}u_{k}r_{k}-d_{22}v_{k})/m_{22}\\
\Delta r_{k} & = (m_{11}-m_{22})u_{k}v_{k}-d_{33}r_{k}+M)/m_{33}\\
\end{aligned}
\right.
\end{equation}

Here, the state $[x_k,y_k,\psi_k]^T$ represents x-axis position, y-axis position and orientation in the $n$-frame, respectively. The state $[u_k,v_k,r_k]^T$ is x-axis velocity, y-axis velocity and angular velocity in the $b$-frame. The control input  $ [F, M]^T$ denotes the force and the torque generated by the thrusters.

By practical system identification of the USV, the system control inputs (i.e., the torque and the force limits of the two thrusts) are constrained by $-19.6N \leqslant F \leqslant 39.2N$ and $-5N m \leqslant M \leqslant 5 N m$, respectively. The system state is constrained by $\left|x\right| \leqslant 70m$, $\left|y\right| \leqslant 70m$, $-1m/s \leqslant u \leqslant 2m/s$, $\left|v\right| \leqslant 1m/s$ and $\left|r\right| \leqslant 0.2rad/s$. $m_{11}= 493.8$, $m_{22}= 455.8$, $m_{33}= 55.8$ are the diagonal elements of the inertia matrix. $d_{11}= 29.2$, $d_{22}= 2173.7$, $d_{33}= 17.7$ are the diagonal elements of the damping matrix.

\subsection{Learning NMPC}\label{subsection:learn NMPC}

\begin{figure}[!t]
\centering
\includegraphics[width=3in]{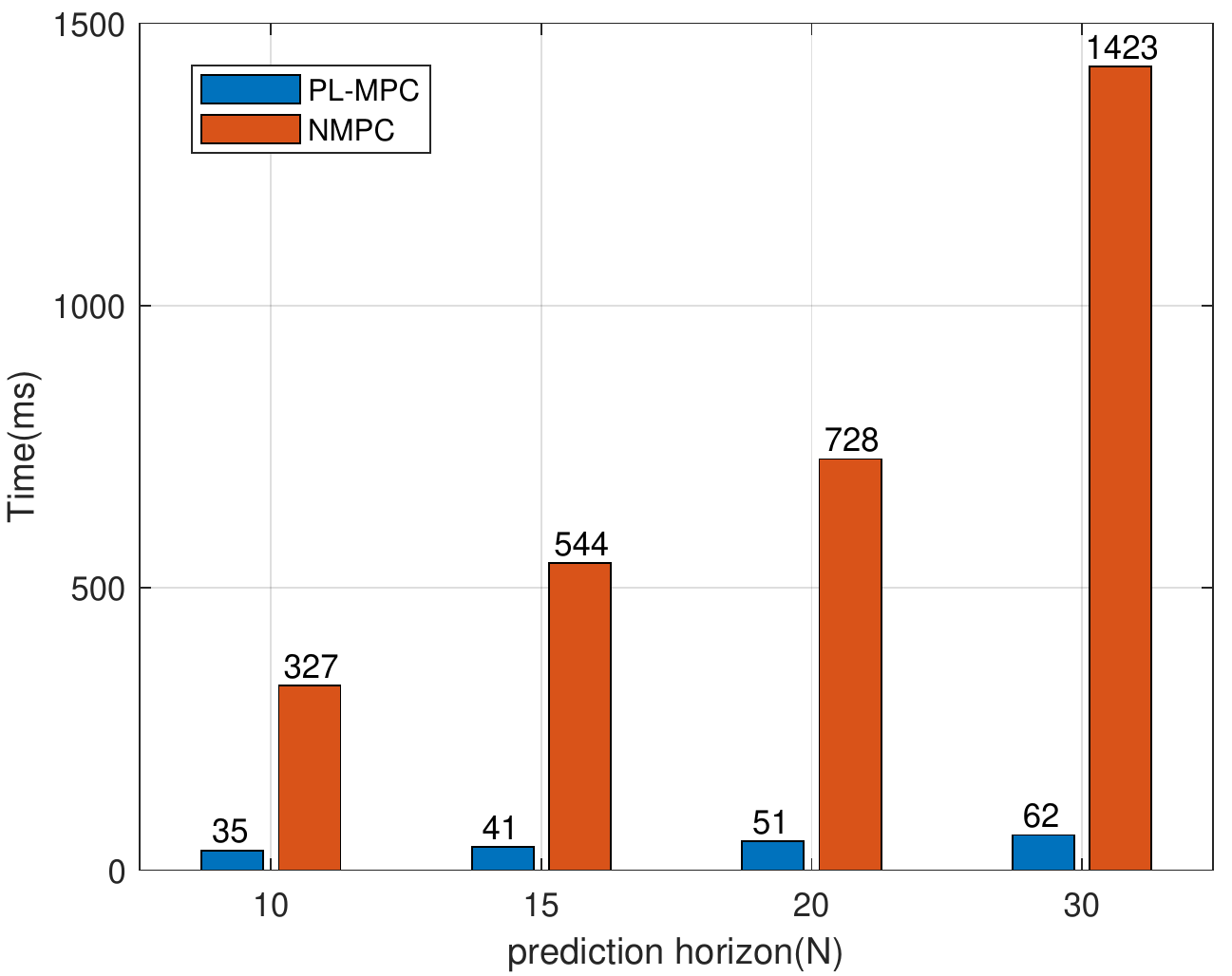}
\caption{Comparison computational time of PL-MPC and NMPC algorithms.}
\label{figure16}
\end{figure}

The training data set is generated by the traditional NMPC algorithm. The detailed parameters are: The three weighting matrices $Q=\rm{diag}(10,10,20,0.1,0.1,0.1)$, $R=\rm{diag}(0.01,0.2)$ and $P=\rm{diag}(10,10,20,0.1,0.1,0.1)$, respectively. The prediction horizon is $N=15$. For such an underactuated USV system, the specific design of the terminal region and terminal control law $\bm{\kappa} ({\bm {x}})$ can refer to \cite{HuipingLi2017ModelPS}. The simulation nonlinear optimization problem of NMPC is solved by CasADi\cite{CasADi}.

For DNN, the node of the input layer is set as $6$, and the node of the output layer is $2N$. There are five layers in the hidden layer. Each layer has $150$, $250$, $250$, $250$ and $50$ nodes, and the activation function is chosen as the ReLU function. We set $\delta=0.01$, $\varepsilon=0.05$ and let $\mathcal{D}$ contain $10^6$ training data, i.e., $M=10^6$. We approximate $E[\mathcal{L}^2({\mathbf {u_\theta}(\bm{x})}, {\mathbf {u^*}(\bm{x})})]$ and $E[\mathcal{L}({\mathbf {u_\theta}(\bm{x})}, {\mathbf {u^*}(\bm{x})})]$ through the test set $\mathcal{T}$ containing $m=10^5$ data.

To verify the learning effect, we apply the PL-MPC algorithm to the simulation model (\ref{example_system}) and compare it with the NMPC algorithm. Fig. \ref{figure1} and Fig. \ref{figure2} present the simulation trajectories with the initial system state $\bm x(0) =[-64,-64,0,0,0,0]^T$, which verifies that the closed-loop system is stable. Fig. \ref{figure3} and Fig. \ref{figure4} compare the system state curves of the NMPC algorithm and the proposed PL-MPC algorithm. It can be observed that the difference of the two algorithms is very small. Fig. \ref{figure17} and Fig. \ref{figure18} show the optimal cost function of the NMPC algorithm and the suboptimal cost function of the PL-MPC algorithm, from which it can be seen that the two algorithms have almost the same cost function.

Simulations of PL-MPC and NMPC with different prediction horizons are also carried out for comparison in the same environment (CPU i7-8550U). By choosing different prediction horizons of MPC, Fig. \ref{figure16} shows the detailed one-step computational time of NMPC and PL-MPC. It can be seen that, the PL-MPC algorithm greatly reduces the computation time compared with the NMPC algorithm, especially with the increase of the prediction horizon.

\subsection{Hardware Implementation}\label{subsection:hardware experiment}

In this section, we implement the designed PL-MPC algorithm to the USV platform and conduct the lake experiments. The experimental equipment and the environment are shown in Fig. \ref{figure5}. In order to verify the effectiveness of the algorithm, we have conducted two experiments corresponding to two control missions: point stabilization and trajectory tracking.

The experimental parameters of point stabilization are the same as those of simulation experiments. By training the policy using the DNN offline with data generated by simulation and deploying it for the USV, we conduct the point stabilization experiments for USV in the lake. Fig. \ref {figure6} is the trajectory of USV point stabilization experiment on the map, and Fig. \ref {figure7} shows the execution process of the proposed method in the experiment.

\begin{figure}[!t]
\centering
\includegraphics[width=3.3in]{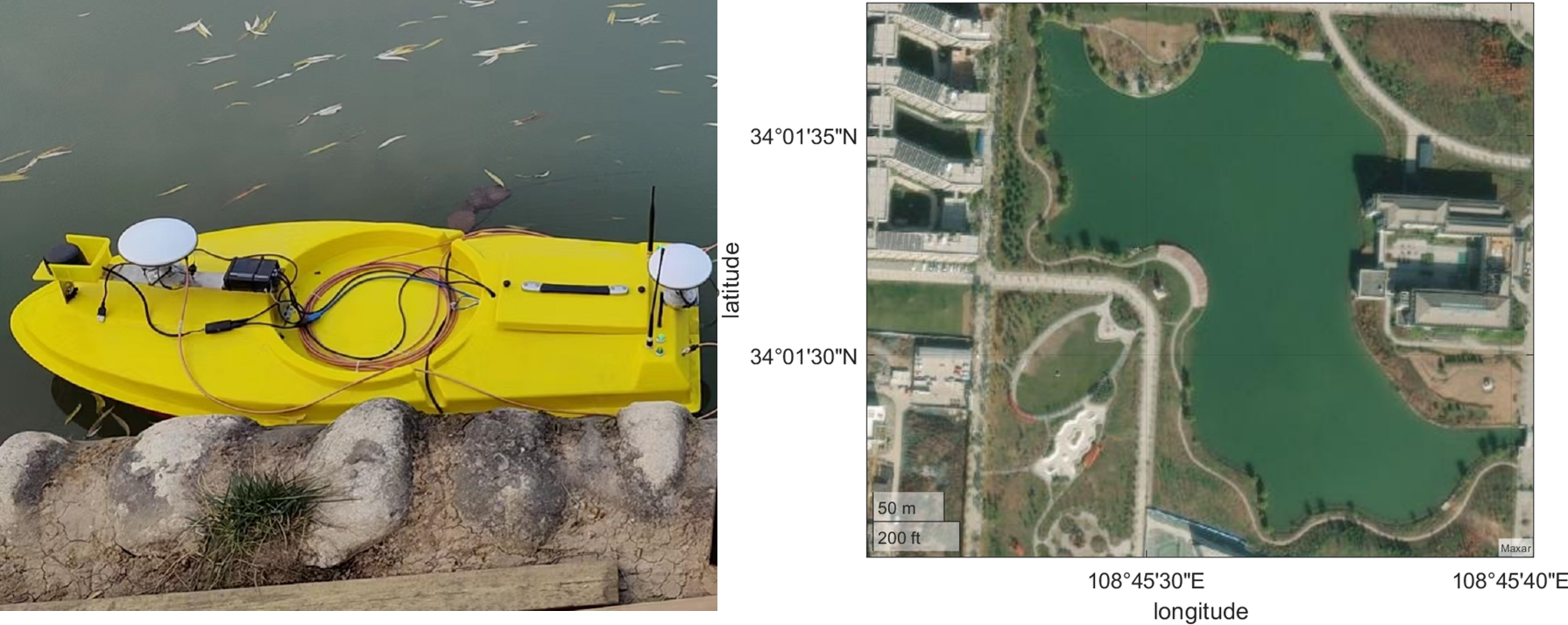}
\caption{USV for experimental and satellite map of experimental environment.}
\label{figure5}
\end{figure}

Fig. \ref {figure8} shows the USV position in the experiment. Under the control algorithm, the USV can move to the equilibrium point and remain stable. Figs. \ref {figure9} - \ref {figure10} show the position error and velocity error, respectively. As shown in Table \ref{tab:table1}, the root mean square (RMS) values of position error are $0.151 m$ and $0.442m$, and the RMS value of angle error is $0.043rad$. The RMS values of velocity error are $0.076m/s$, $0.019m/s$ and $0.036rad/s$, respectively. Note that there exist unavoidably disturbances from waves and winds in the lake when the experiments were conducted. Even though the unknown disturbances, the PL-MPC controller has achieved satisfactory control performance in terms of the relatively small errors in real time experiments.

\begin{figure*}[!t]
\centering
\subfloat[]{\includegraphics[width=3in]{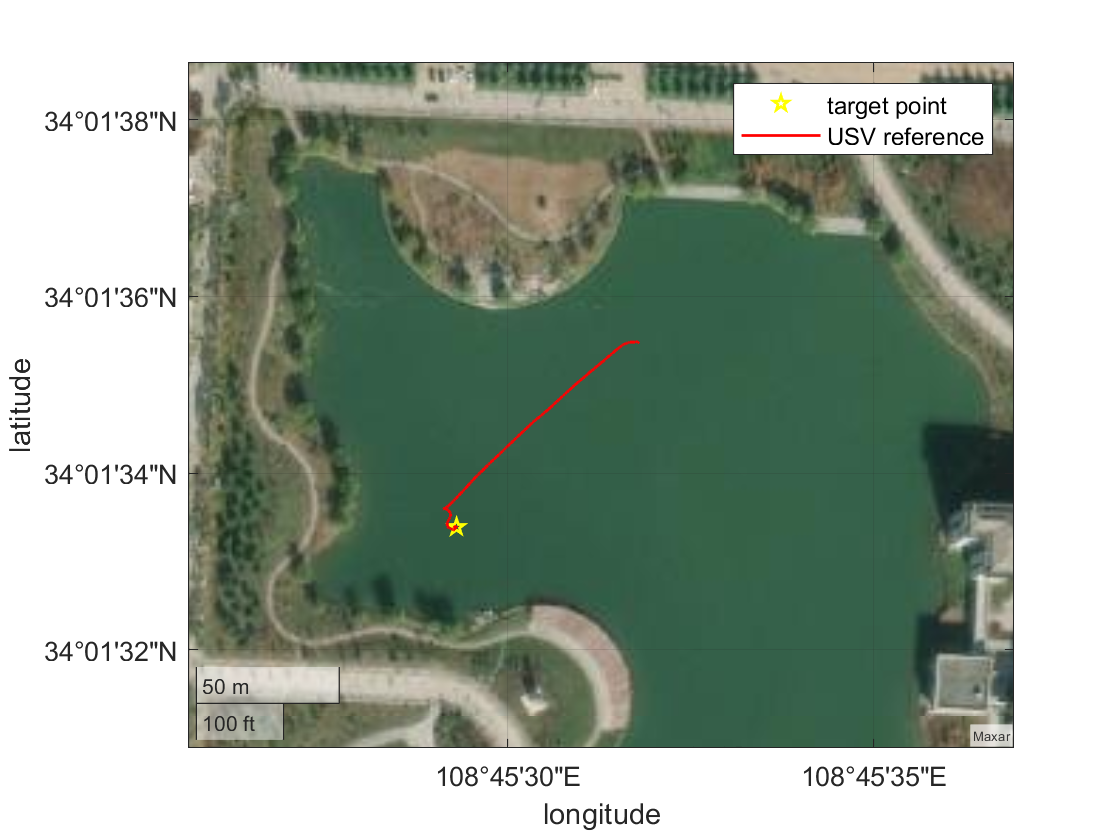}
\label{figure6}}
\hfil
\subfloat[]{\includegraphics[width=3.3in]{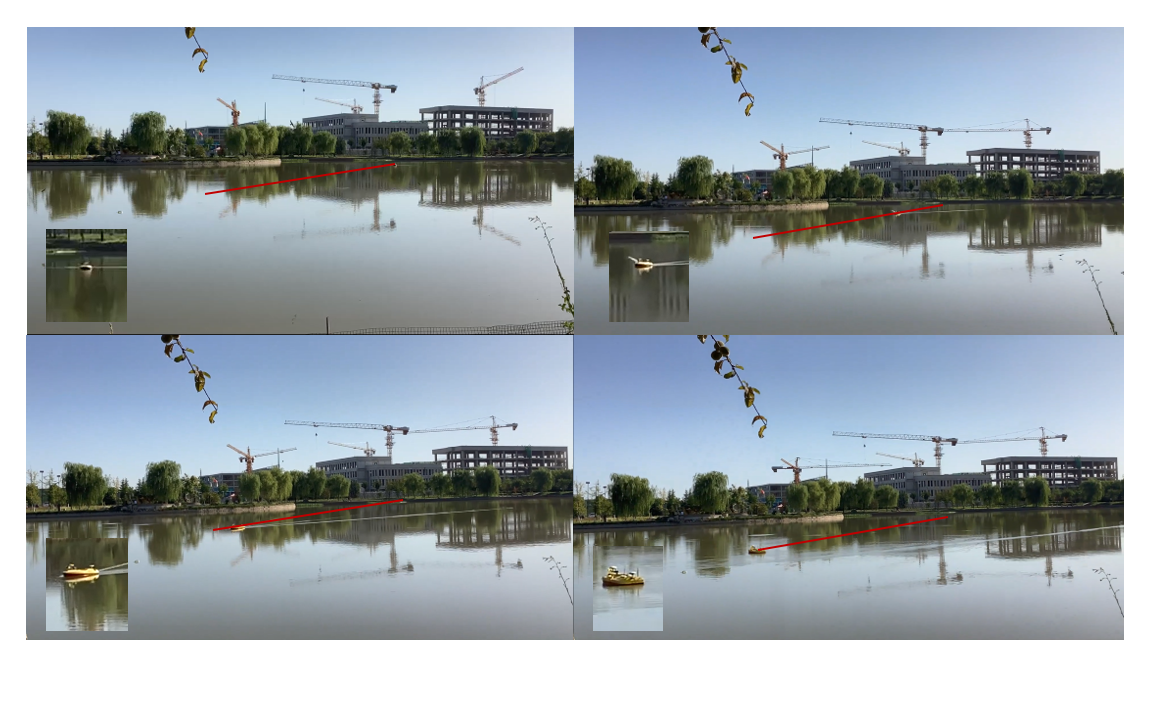}
\label{figure7}}
\caption{USV point stabilization experiment and trajectory. (a) USV trajectory. (b) Experimental picture.}
\end{figure*}

\begin{figure}[!t]
\centering
\includegraphics[width=3in]{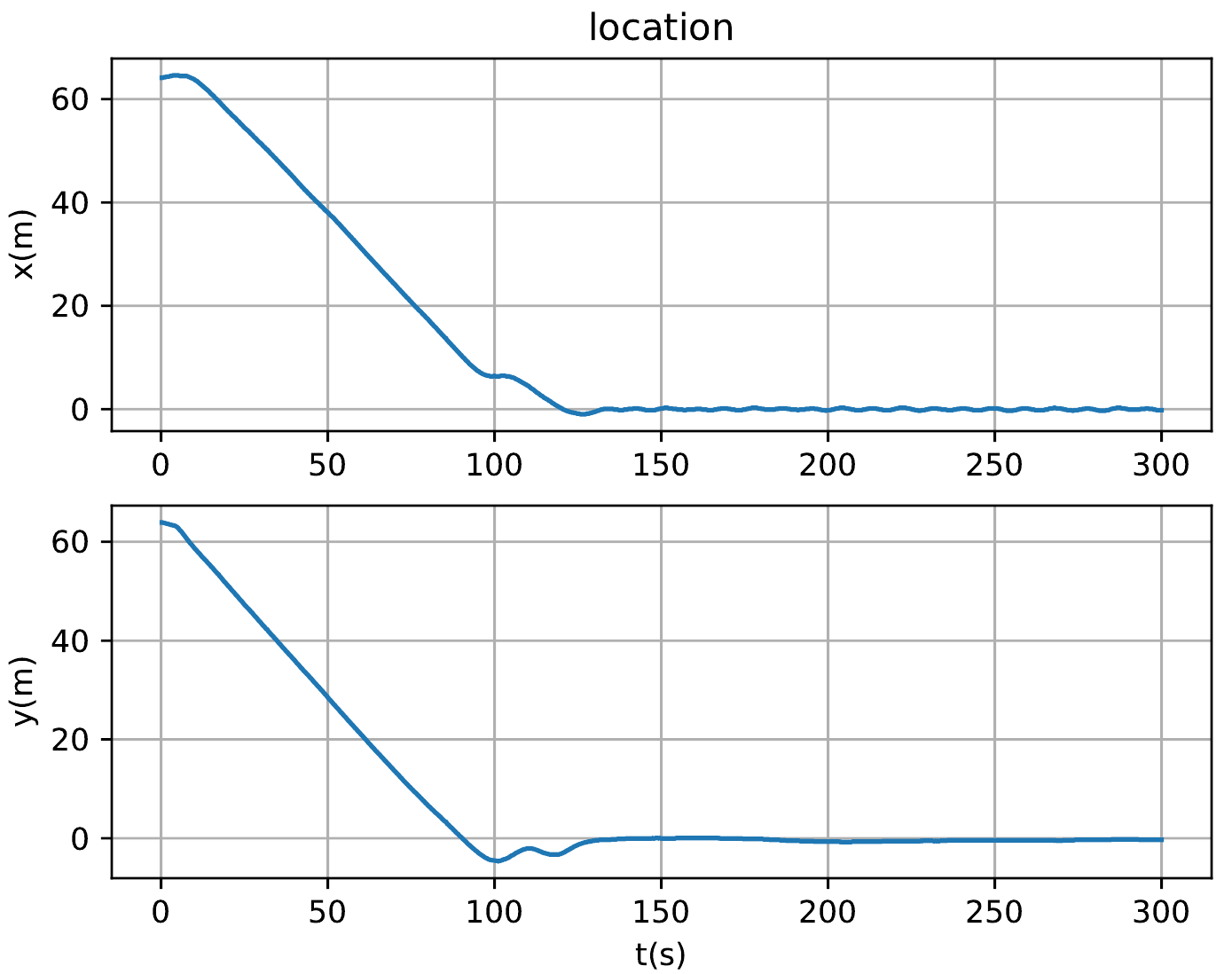}
\caption{Position trajectory of point stabilization experiment with initial position $[64, 64]^T$.}
\label{figure8}
\end{figure}

\begin{figure}[!t]
\centering
\includegraphics[width=3in]{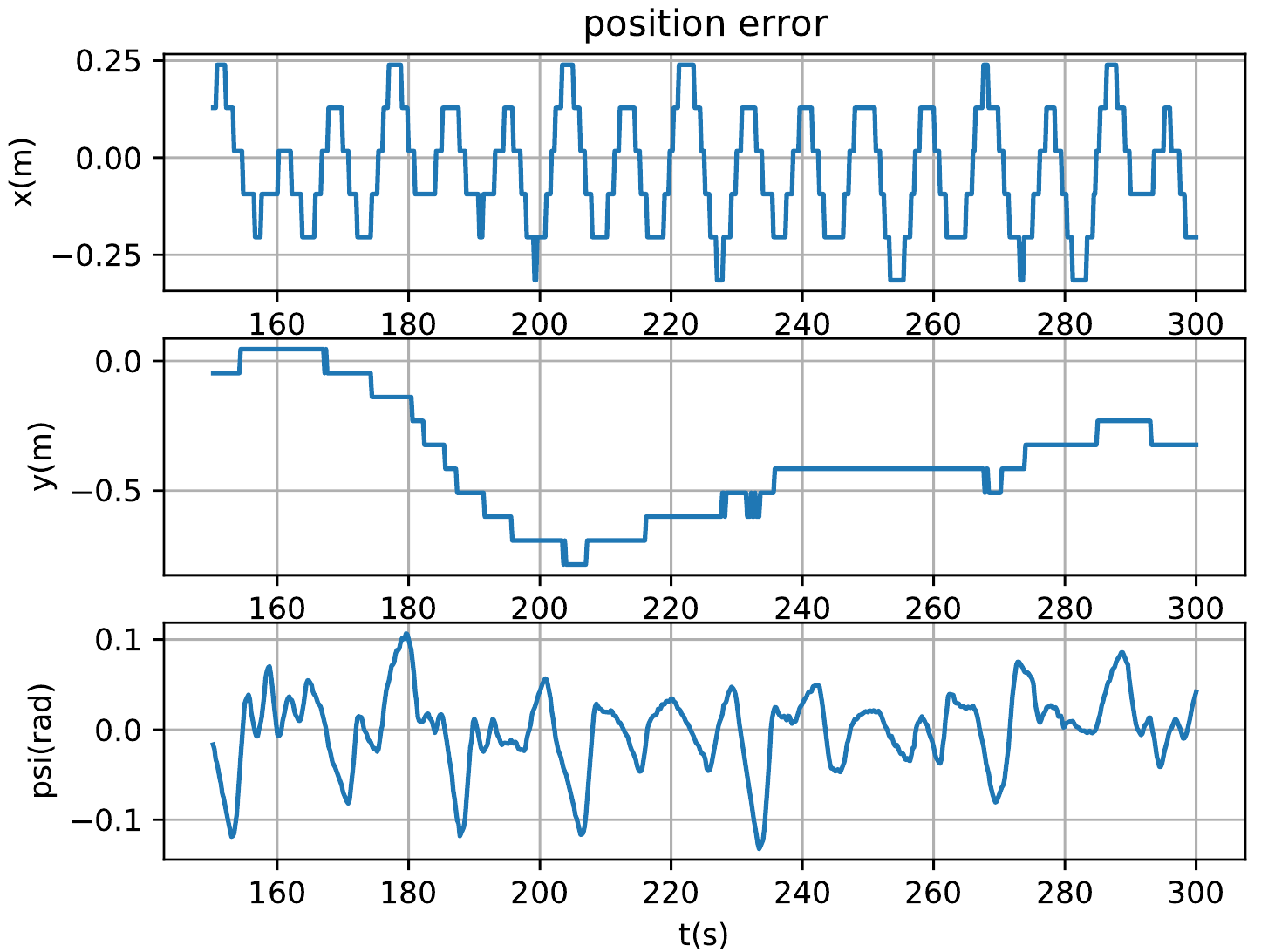}
\caption{Position error of point stabilization experiment.}
\label{figure9}
\end{figure}

\begin{figure}[!t]
\centering
\includegraphics[width=3in]{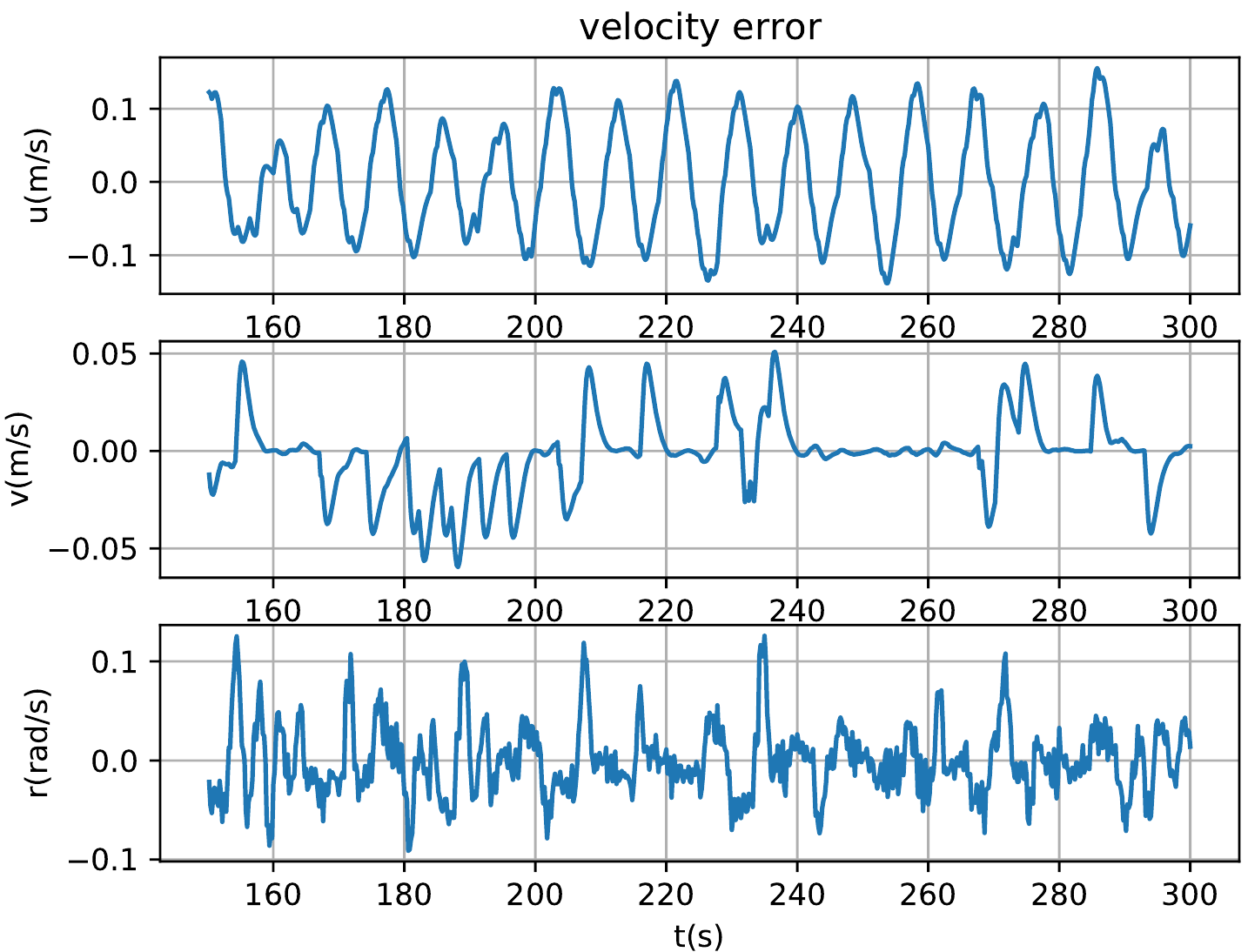}
\caption{Velocity error of point stabilization experiment.}
\label{figure10}
\end{figure}

\begin{table}[!t]
\caption{RMS Values of State Errors\label{tab:table1}}
\centering
\begin{tabular}{|c||c|c|c|c|c|c|}
\hline
state error & $x$ & $y$ & $\psi$ & $u$ & $v$ & $r$\\
\hline
RMS vlaue & $0.151$ & $0.442$ & $0.043$ & $0.076$ & $0.019$ & $0.036$\\
\hline
unit & $m$ & $m$ & $rad$ & $m/s$ & $m/s$ & $rad/s$\\
\hline
\end{tabular}
\end{table}

For trajectory tracking experiments, we design a reference trajectory $(\bm {x}_d(k), \bm {u}_d(k))$ as a circle. The reference trajectory and USV share the same kinetic model in (\ref{example_system})-(\ref{example_system2}), it is denoted as $\bm{x}_d(k+1) = \bm{f} (\bm{x}_d(k),\bm{u}_d(k))$. Here, $\bm{x}_d(k)=[x_d(k), y_d(k),\psi(k), u_d(k), v(k), r(k)]^T$ and $\bm {u}_d(k)=[F_d(k),M_d(k)]^T$. In this experiment, the reference force and moment are set as: $F_d=17.5N$ and $M_d=1Nm$, respectively. The reference trajectory under this control input is a circle, as shown in Fig. \ref {figure13}.

To track this trajectory, we define the difference between the system state and the reference trajectory state as $\bm{x}_e(k)=\bm{x}(k)-\bm{x}_d(k)$ and the difference between the system control input and the reference input as $\bm{u}_e(k)=\bm{u}(k)-\bm{u}_d(k)$.

We use $\bm{x}_e$ and $\bm{u}_e$ instead of $\bm x$ and $\bm u$ as the NMPC cost function optimization variables. Naturally, the input of the neural network changes to $\bm {x}_e(k)$ and the output changes to $\bm{U}_e(k)$, where $\bm{U}_e(k)=[\bm{u}_e(k;k),\bm{u}_e(k+1;k),...,\bm{u}_e(k+N-1;k)]$. The three weighting matrices $Q=diag(q_{11},q_{22},q_{33},q_{44},q_{55},q_{66})=diag(10,10,0.1,1,1,1)$, $R=diag(r_{11},r_{22})=diag(0.01,0.01)$ and $P=diag(10,10,0.1,1,1,1)$, respectively. The MPC constraints and neural network parameters are the same as the simulation experiments. The terminal region and terminal control law are designed as reference \cite{AUV}.
We deploy the designed DNN-based policy into the trajectory control of USV, and the trajectory tracking control algorithm is executed in real time with a sampling rate at $5Hz$.

\begin{figure*}[!t]
\centering
\subfloat[]{\includegraphics[width=2.5in]{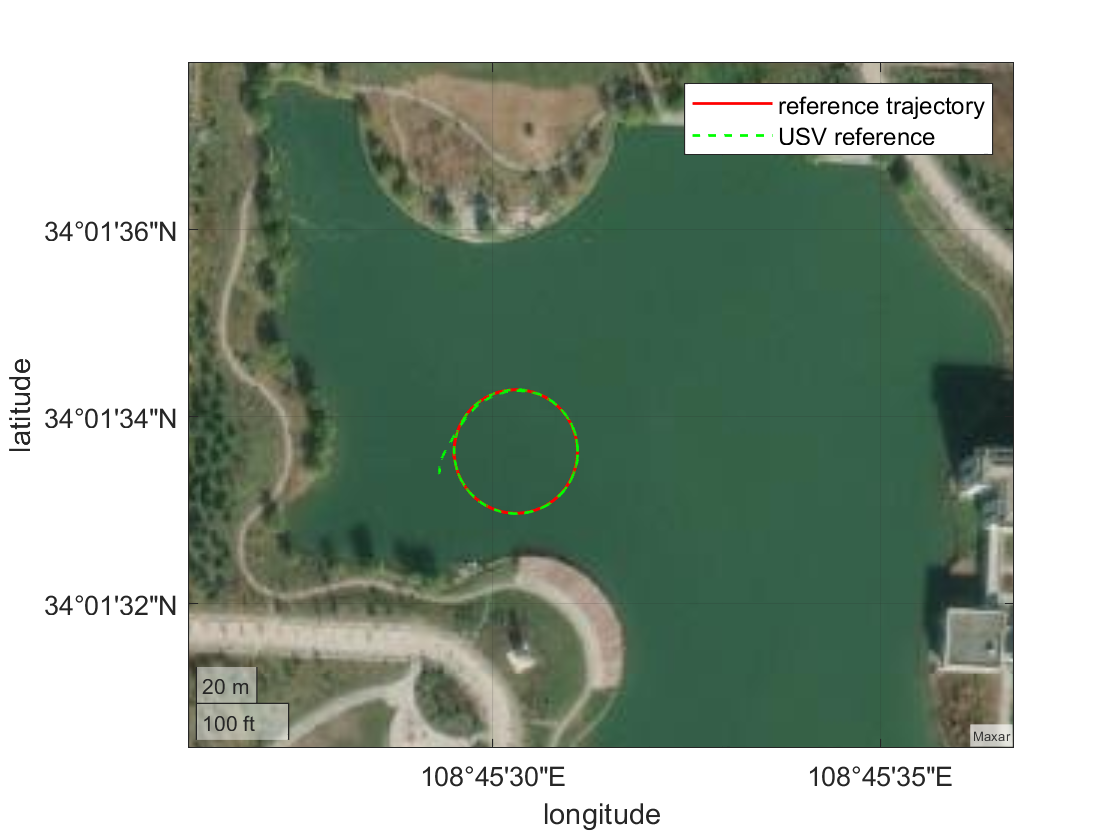}
\label{figure11}}
\hfil
\subfloat[]{\includegraphics[width=4in]{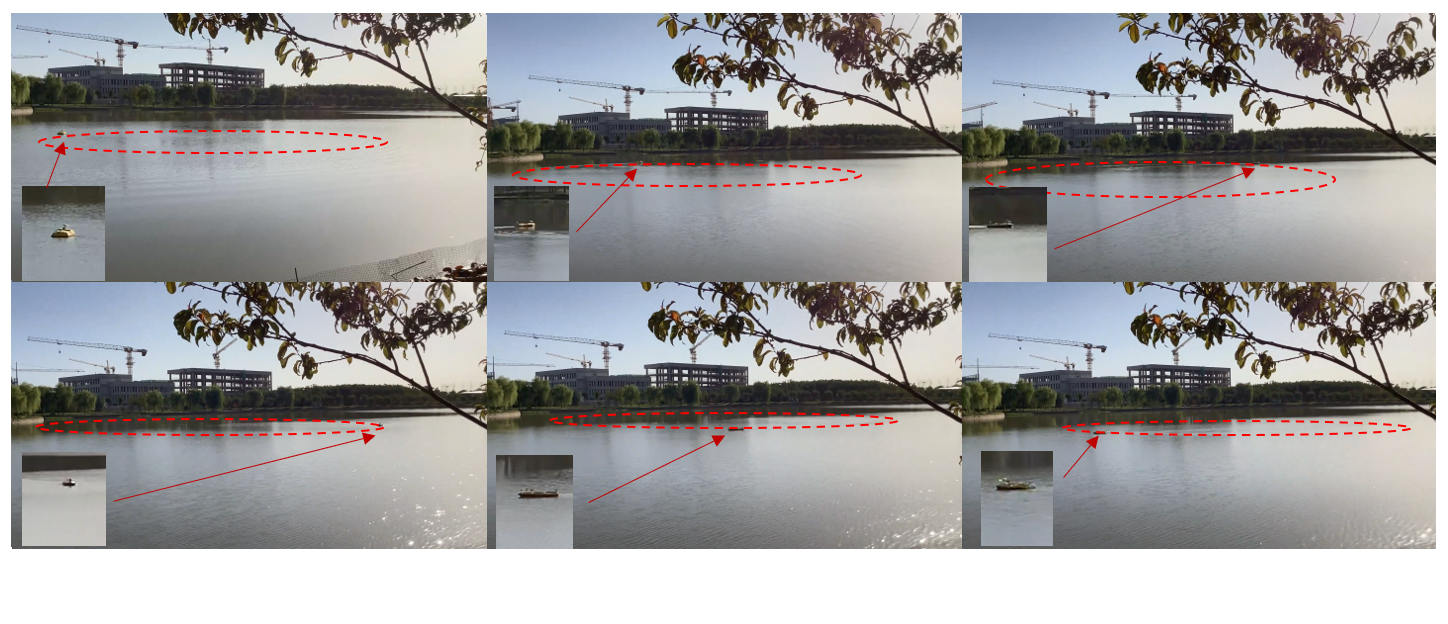}
\label{figure12}}
\caption{USV trajectory tracking experiment and trajectory.  (a) USV trajectory. (b) Experimental picture.}
\end{figure*}

\begin{figure}[!t]
\centering
\includegraphics[width=3in]{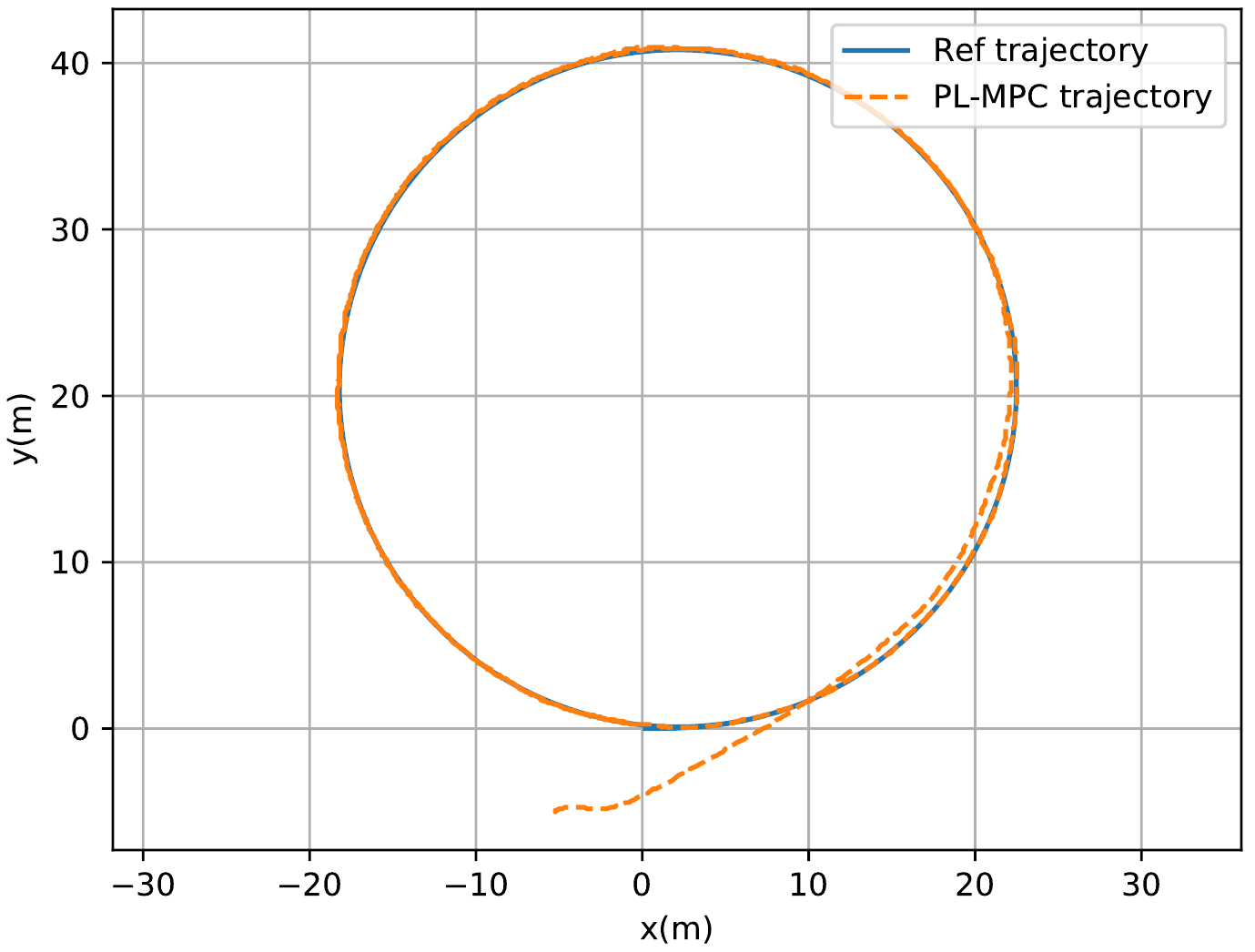}
\caption{Trajectory of USV with initial position $[-5,-5]^T$.}
\label{figure13}
\end{figure}

\begin{figure}[!t]
\centering
\includegraphics[width=3in]{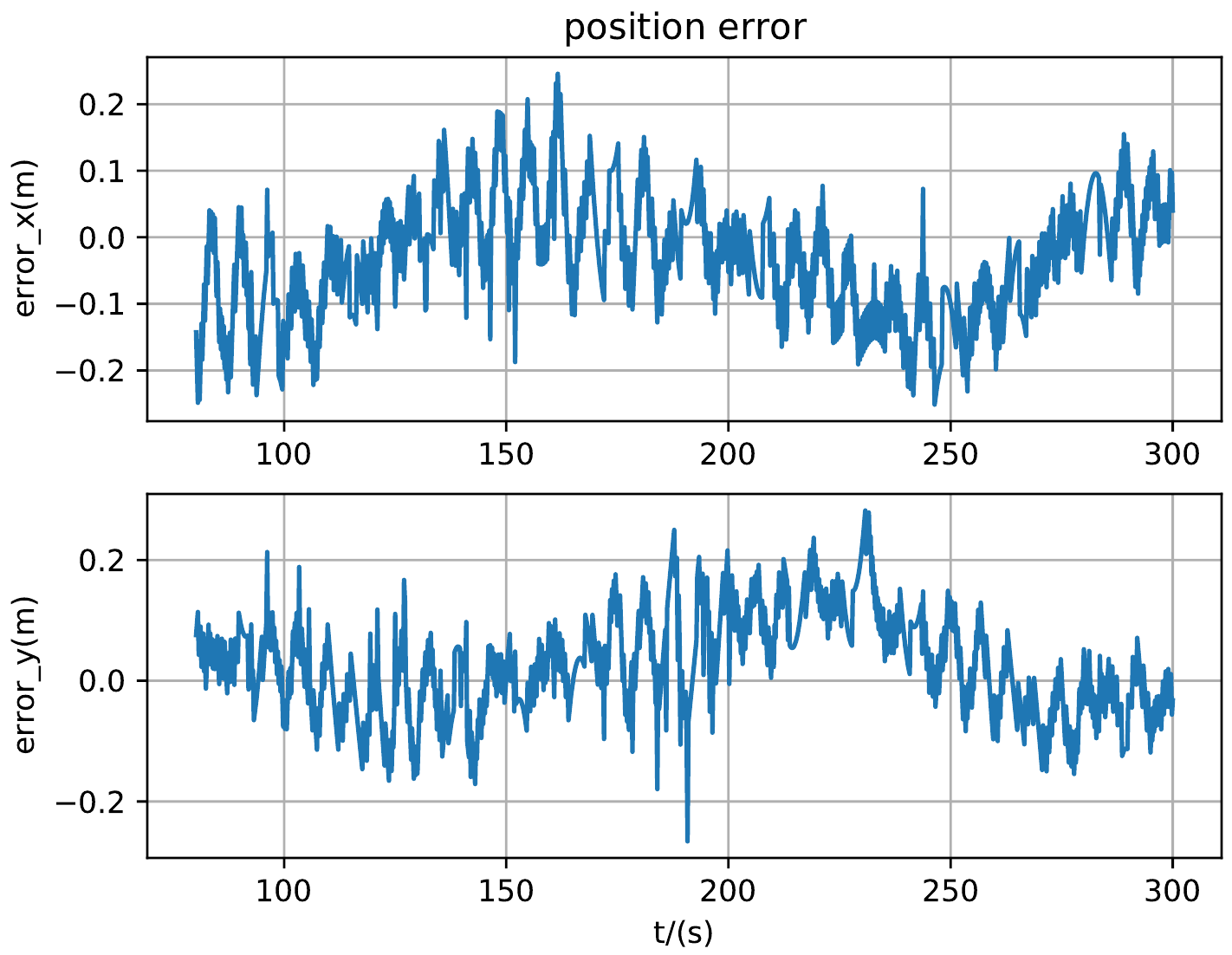}
\caption{Position error of trajrctory tracking experiment.}
\label{figure14}
\end{figure}
Fig. \ref {figure11} is the trajectory of USV point stabilization experiment on the map, and Fig. \ref {figure12} shows the execution process of the proposed method in the experiment.
Fig. \ref {figure13} shows the USV position trajectory and reference trajectory, from which it can be seen that the proposed PL-MPC algorithm provides excellent tracking performance. The position errors are shown in Fig. \ref {figure14} and the velocity errors are shown in Fig. \ref {figure15}. The RMS values of position and velocity error are $0.094m$, $0.088m$ and $0.044m/s$, $0.079m/s$ demonstrated in Table \ref{tab:table2}. It can be observed that the position and velocity errors under the PL-MPC algorithm are small. This verifies the effectiveness of the proposed PL-MPC algorithm in the actual environment.

\begin{figure}[!t]
\centering
\includegraphics[width=3in]{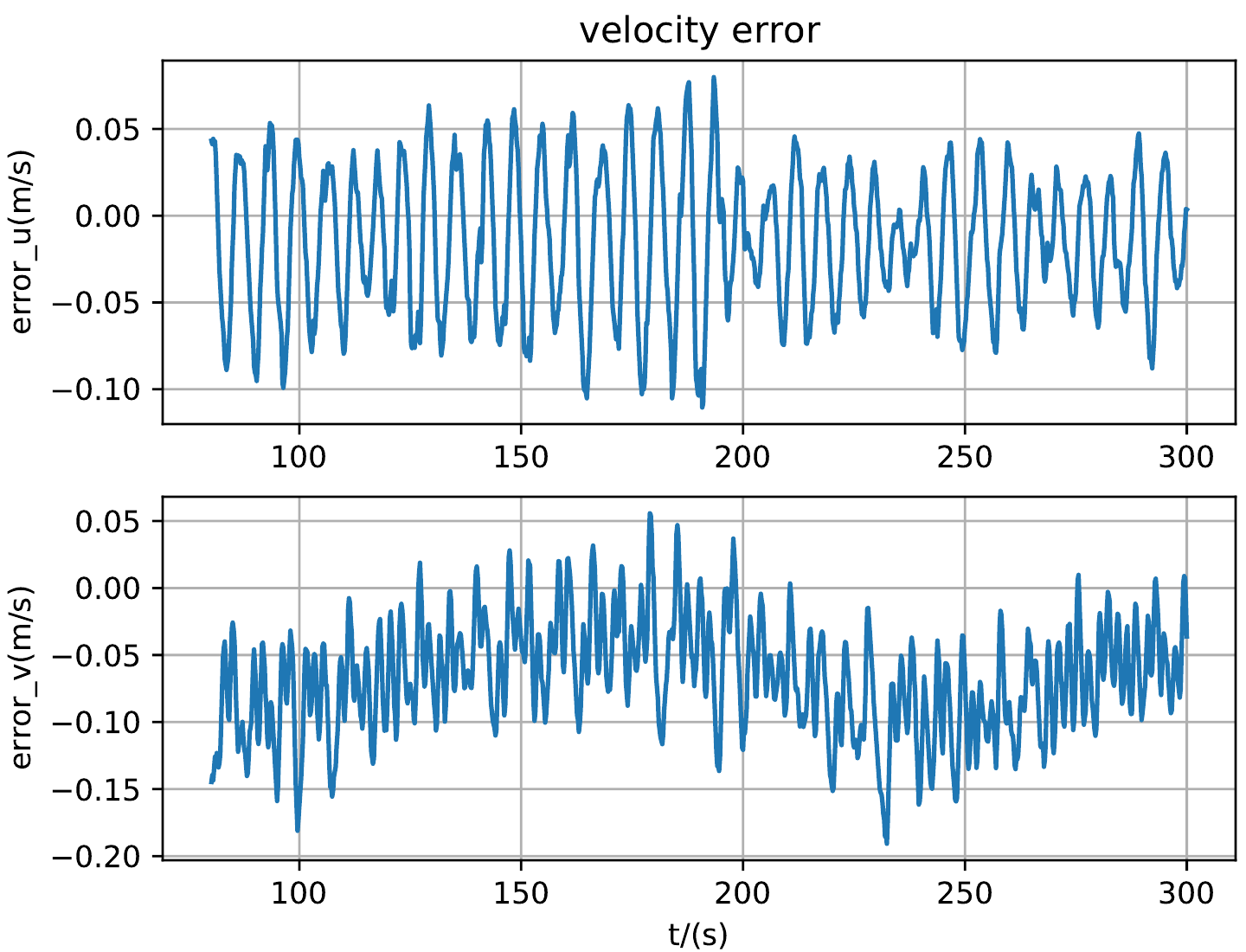}
\caption{Velocity error of trajrctory tracking experiment.}
\label{figure15}
\end{figure}

\begin{table}[!t]
\caption{RMS Values of Position and Velocity Errors\label{tab:table2}}
\centering
\begin{tabular}{|c||c|c|c|c|}
\hline
state error & $x$ & $y$ & $u$ & $v$ \\
\hline
RMS vlaue & $0.094$ & $0.088$ & $0.044$ & $0.079$\\
\hline
unit & $m$ & $m$ & $m/s$ & $m/s$\\
\hline
\end{tabular}
\end{table}

\section{CONCLUSIONS}\label{section:conclusion}

In this paper, we have developed a PL-MPC method for nonlinear discrete-time system with state and control inputs constraints, and implemented it successfully to the motion control of USVs. The DNN is trained to approximate NMPC control policy to accelerate the computational speed. By this method, we have designed the dual optimization learning algorithm, making the approximation control policies satisfy the control input and state constraints. The sufficient conditions on ensuring the closed-loop stability have also been developed. The hardware experiment for the motion control of USV have verified the feasibility and advantages of the proposed method.



\vfill

\end{document}